\documentclass[twoside,11pt]{article}

\usepackage{jmlr2e} % Comment this out to format the paper for a preprint
\usepackage{amsmath}
\usepackage{xcolor} % for using \textcolor
\usepackage{tikz} 
\usepackage{pgfplots} % For Tikz
\pgfplotsset{compat=1.16}
\usepackage{mymacros} % User-defined macro
\usepackage{comment} 
\usepackage[all]{xy}

\usepackage{lastpage}
\jmlrheading{26}{2025}{1-\pageref{LastPage}}{4/24; Revised 6/25}{9/25}{24-0606}{Nihat Ay and Jesse van Oostrum and Adwait Datar}
\ShortHeadings{On the Natural Gradient of the Evidence Lower Bound}{Ay, van Oostrum and Datar}

\firstpageno{1}

\begin{document}

\title{On the Natural Gradient of the Evidence Lower Bound}

\author{\name Nihat Ay \email nihat.ay@tuhh.de \\
       \addr Institute for Data Science Foundations\\
       Hamburg University of Technology\\
       21073 Hamburg, Germany \\
       \\
       \vspace{-0.6cm}\\
       \addr Santa Fe Institute\\
       Santa Fe, NM 87501, USA \\
       \\
       \vspace{-0.6cm}\\
       \addr Leipzig University\\
       04109 Leipzig, Germany
       \AND \\
       \vspace{0.1cm}
       \\
       \name Jesse van Oostrum \email jesse.van@tuhh.de \\
       \addr Institute for Data Science Foundations\\
       Hamburg University of Technology\\
       21073 Hamburg, Germany
       \AND \\
       \vspace{0.1cm}
       \\
       \name Adwait Datar \email adwait.datar@tuhh.de \\
       \addr Institute for Data Science Foundations\\
       Hamburg University of Technology\\
       21073 Hamburg, Germany}

\editor{Dan Alistarh}

\maketitle

\begin{abstract}%   <- trailing '%' for backward compatibility of .sty file

This article studies 
the Fisher-Rao gradient, also referred to as the natural gradient, of the evidence lower bound (ELBO) which plays a central role in generative machine learning. It reveals that the gap between the evidence and its lower bound, the ELBO, has essentially a vanishing natural gradient within unconstrained optimization. 
As a result, maximization of the ELBO is equivalent to minimization of the Kullback-Leibler divergence from a target distribution, the primary objective function of learning.
Building on this insight, we derive a condition under which this equivalence persists even when optimization is constrained to a model. This condition yields a geometric characterization, which we formalize through the notion of a \textit{cylindrical model\/}.

\end{abstract}
\medskip

\begin{keywords}
  Evidence lower bound, variational gap, natural gradient, information geometry, variational inference
\end{keywords}

\section{Introduction}
\label{submission}
Generating samples from a complex target probability distribution represents the key challenge of generative machine learning. Typical examples of such a distribution are given in terms of natural images or token sequences in large language models. 
A primary objective function for training a generative network is based on the log-likelihood of samples, referred to as the {\em evidence\/}. In order to train a generative network, a corresponding recognition network has to be trained, with which the evidence is replaced as an objective function by the {\em evidence lower bound\/} ({\em ELBO\/}). 
This bound has its roots in variational methods, originally developed by Feynman \citep[see, for example, Chapter 3, Section 3.4 of][]{feynman1972statistical} in the context of statistical physics, where it was employed to approximate free energy.
These variational methods have since been adapted for statistical inference and machine learning, proving especially useful in the formulation of the Helmholtz Machine \citep{dayan1995helmholtz, ikeda1998convergence} and other early applications \citep{hinton1993keeping, hinton1993autoencoders, mackay1995developments}.
In more recent applications, the ELBO has become a core objective for training deep generative models, such as the Variational Autoencoder (VAE) \citep{kingma2013auto} and other deep generative models \citep{rezende2014stochastic}. Beyond machine learning, the ELBO also holds a central role in cognitive science and neuroscience, underpinning the Free Energy Principle \citep{friston2005theory}.
More recently, a generalization of the ELBO called the \textit{generalized evidence lower bound} (GLBO) has been proposed for model selection to ensure better generalization  \citep{chen2018variational}.
A closely related idea of using a so-called Stein gradient instead of the usual gradient of the lower bound has been pursued in \citep{pu2017vae}.
As a generalization of the Kullback-Leibler divergence, R{\'e}nyi's $\alpha$-divergences (related to but different from the $\alpha$-divergence in information geometry) have been studied in \citep{li2016renyi} where a smooth interpolation between the evidence lower bound and the log (marginal) likelihood is formalized via the parameter $\alpha$ thereby unifying a number previously existing approaches.
\medskip

Intuitively, one should expect that the gap between the evidence and its lower bound, the so-called {\em variational gap\/}, crucially affects the quality of learning. Various components of the variational gap and their influence on the learning have been studied, including the {\em approximation gap\/}, the {\em amortisation gap\/}, and the {\em conditioning gap\/} \citep{bayer2021mind}.
While tightening the bound appears to be beneficial at first sight, it has also been observed that a tighter bound does not necessarily imply an improvement and can even compromise the objective of learning
\citep{Rainforth2018TighterVB}.  
Aiming at an explanation of this phenomenon, our article is based on the simple idea that learning in terms of gradient methods is not so much dependent on the variational gap itself but on its gradient. While the variational gap can be rather large, its gradient might vanish (in  a particular sense that we are going to specify) and therefore has no effect on the learning. 
We will pursue this idea with the help of information geometry \citep{amari2000methods, amari2016information, ay2017information}, a framework that is particularly appropriate for analyzing the evidence lower bound and the variational gap. In particular, we will study the natural gradient \citep{amari1998} of both quantities, that is the gradient with respect to the Fisher-Rao metric which we will introduce below. Our analysis will reveal a geometric criterion for the variational gap to have no effect on the learning.
This core result depends crucially on the information-geometric structures and does not hold for the standard Euclidean geometry that underlies most existing gradient-based algorithms.
In what follows, we briefly outline the framework of information geometry. 
\medskip

Originating from statistics, information geometry provides efficient methods for the field of machine learning which are based on duality concepts from differential geometry \citep{amari2000methods, amari2016information, ay2017information}. Most prominently, it suggests as a fundamental structure a Riemannian manifold $({\mathcal P},g)$, equipped with a pair $(\nabla, \nabla^\ast)$ 
of affine connections that are dual with respect to the Riemannian metric $g$. A particularly important situation is given when the two 
connections are flat, which implies the existence of a pair of dual affine coordinate systems and a corresponding canonical divergence 
$D: {\mathcal P} \times {\mathcal P} \to {\mathbb R}_+$. 
In this case, the geometry is comparable with the Euclidean geometry of ${\Bbb R}^d$, with $D$ being proportional to the standard squared distance function in ${\Bbb R}^d$.
These structures can lead to  highly efficient learning algorithms when consistently used together. 
To be more precise, the distinguished canonical divergence $D$ offers a natural way to define an objective or risk function $\mathcal{L}:\cP \rightarrow \mathbb{R}$ for learning. 
When optimizing this divergence in terms of the gradient descent method, the Riemannian metric $g$ should be applied to define the natural gradient ${\rm grad}_p \mathcal{L}$ in $p\in \mathcal{P}$ via the equation
\begin{equation}\label{eq:nat_grad_defn}
    d\mathcal{L}_p(A) \, =\, g_p({\rm grad}_p \mathcal{L},A)
\end{equation}
for all tangent vectors $A$ in the tangent space $T_p \mathcal{P}$.
This leads to the {\em natural gradient method\/} which plays a crucial role in the theory of neural networks and machine learning \citep{amari1998, ollivier2015, martens2020new}. With these choices, the learning trajectories are then simply straight lines in the above-mentioned dual affine coordinate systems. Loosely speaking, the learning converges to a solution in the most direct way \citep{fujiwara1995gradient, datar2025convergence}. This demonstrates the simplicity and efficiency of learning as a result of a consistent combination of the underlying geometric structures.  
\medskip

Despite the great advantages of the outlined information-geometric approach to learning, it is a highly non-trivial task to actually utilize and implement this approach within the setting of machine learning. In what follows, we highlight two complications that are particularly relevant for this article. 
\begin{enumerate}
\item The manifold    
${\mathcal P}$ of the above paragraph plays the role of a high-dimensional ambient space, equipped with a 
dually flat structure $g$, $\nabla$, and $\nabla^\ast$.  
Thus, it comes with a canonical divergence for learning, as outlined above. 
The learning, however, is typically restricted to a lower-dimensional model 
${\mathcal M} \subseteq {\mathcal P}$. The restriction of the convenient geometric structures on ${\mathcal P}$ to the model ${\mathcal M}$ is typically much more complex. Only in exceptional cases, this restriction preserves the simplicity of the geometry of ${\mathcal P}$.

\item In addition to that, we face another potential source of complication. 
Typically, the expressive power of a learning system has to be increased in terms of a set of latent or hidden units denoted by $H$. 
In this case, the primary model for learning is associated with the observed or visible units denoted by $V$. It is obtained as the image ${\mathcal M}_V$ of a model ${\mathcal M}$ under the marginalization map. 
Even if ${\mathcal M}$ inherits geometric properties from its ambient space ${\mathcal P}$ that are advantageous for learning, these properties need not be preserved under this marginalization.  
\end{enumerate}
\medskip

To summarize, we face two sources of complexity when designing information-geometric learning algorithms, the restriction of natural structures from the ambient space ${\mathcal P}$ to the model ${\mathcal M}$, and the marginalization which maps ${\mathcal M}$ to the model ${\mathcal M}_V$.
In this article, we aim to disentangle  the individual complexities resulting from these two operations by studying the optimization processes first on $\cP$ and then extend the analysis to the constrained setting $\cM$. We follow this reasoning in order to discuss the evidence lower bound and the variational gap from an information-geometric perspective. We relate the maximization of the evidence to the maximization of its lower bound in view of information geometry and highlight the simplicity and consistency of both optimization problems when considered in the full ambient space, without restricting it to a model ${\mathcal M}$. 
We show that in this case the evidence lower bound leads to the same natural gradient field as the original objective function, the evidence, which we find remarkable.
This equivalence is not necessarily preserved when restricting the optimization to a model ${\mathcal M}$. We provide a sufficient condition for this to hold, which requires the notion of a cylindrical model.    
\medskip 

In this article, we follow two story lines, one referring to the evidence and its lower bound and one referring to corresponding Kullback-Leibler divergences.
We use the former story line to formulate the main problem and to convey the key findings without assuming a background in information geometry.
The latter story line is more convenient for our information-geometric studies.  
Section \ref{target} introduces the primary objective of learning, minimizing the Kullback-Leibler divergence from a target distribution on states of the visible units, and briefly outlines its relation to the evidence and its lower bound. This section is generally accessible, without a background in information geometry.
In Section \ref{basic}, we are then going to review basic information-geometric structures, thereby introducing the notation used in this article. 
This section also includes results from the previous work \citep{ay2020locality} on which this article is based. Section \ref{extendedpr} deals with the analysis of the optimization problem for the extended system, including visible and hidden units, and relates it to the primary optimization problem defined for its visible part. Section \ref{natelbo} relates these results to the evidence and its lower bound, thereby making statements on their respective natural gradients. Section \ref{bayesnet} concludes with a result that is particularly helpful when dealing specifically with Bayesian graphical models.   

\section{Learning a Target Distribution and the Evidence Lower Bound} \label{target}
Throughout this article, 
we consider a system consisting of visible units $V$ and hidden units $H$
taking values in state sets $\mathsf{X}_V$ and $\mathsf{X}_H$, respectively.
For simplicity, we assume $\mathsf{X}_V$ and $\mathsf{X}_H$ to be finite.
The set of all strictly positive probability distributions on joint states $(x_V, x_H)$ is denoted by ${\mathcal P}_{V,H}$, which we also abbreviate as ${\mathcal P}$. In order to study learning in terms of the natural gradient method, we consider a model ${\mathcal M}$ consisting of probability distributions $p_\theta(x_V,x_H) = p(x_V,x_H; \theta)$ which are parametrized by a parameter vector $\theta$ in ${\Bbb R}^d$. 
Typically, the parameter set is an open subset $\Theta$ of 
${\mathbb R}^d$, and we obtain ${\mathcal M}$ as the image of the parametrization  
\begin{equation} \label{parametr}
 \varphi: \; \Theta \; \rightarrow \;  {\mathcal M} \subseteq {\mathcal P}, \qquad  \theta \; \mapsto \; p_\theta. 
\end{equation}
The model $\cM$ is referred to as a 
{\em generative model\/}. The objective of learning is to generate a probability distribution on visible states $x_V$ that is close to some target distribution. Here, we interpret the hidden units merely as auxiliary units to increase the expressive power. The learning objective should therefore only refer to the visible units. To be more precise, we denote by ${\mathcal P}_V$ the set of strictly positive probability distributions on states $x_V$ and consider the natural marginalization map 
\[
    \pi_V : {\mathcal P} \to {\mathcal P}_V, 
\]
which assigns to a joint probability distribution $p(x_V, x_H)$ the marginal distribution 
\begin{equation} \label{margina}
      p(x_V) := \sum_{x_H} p(x_V, x_H ).  
\end{equation}    
The image of the model ${\mathcal M}$, that is $\pi_V({\mathcal M})$, is denoted by ${\mathcal M}_V$. It consists of all probability distributions 
that can be generated by the learning system. With the parametrization (\ref{parametr}), we can parametrize ${\mathcal M}_V$ in terms of 
\[ 
    \pi_V\circ \varphi: \; \Theta \; \rightarrow \;  {\mathcal M}_V \subseteq {\mathcal P}_V, \qquad  \theta \; \mapsto \; \pi_V(p_\theta). 
\]
In this article, we will mostly omit the parameter and simply write $p \in {\mathcal M}$ and $p \in {\mathcal M}_V$, respectively.
\medskip

Now consider a target distribution $p^\ast \in {\mathcal P}_V$. The objective of learning is to find $p \in {\mathcal M}_V$ that is close to $p^\ast$.
To achieve that, we minimize the KL-divergence of  
$p^\ast$ from a distribution $p \in {\mathcal M}_V$, that is, 
\begin{equation}
    D(p^\ast \| p ) 
       \; := \; \sum_{{x}_V} p^\ast({x}_V) \ln \frac{p^\ast({x}_V)}{p({x}_V)}.   \label{lo}
\end{equation}  
Throughout this article, we refer to this function as a primary objective function defined on ${\mathcal M}_V$ and therefore only involving visible units. This will be compared with corresponding lifted objective functions defined on ${\mathcal M}$ which involve the visible as well as the hidden units.
Observe that the minimization of $D(p^\ast \| \cdot)$
is equivalent to the minimization of the 
{\em cross entropy\/}
\begin{equation*} %\label{cross}
   - \sum_{x_V} p^\ast(x_V) \ln p(x_V)
\end{equation*}
because these two functions differ only by a constant, the {\em entropy\/} of $p^\ast$, which is given by 
\begin{equation*}
   - \sum_{x_V} p^\ast(x_V) 
    \ln p^\ast(x_V). 
\end{equation*}
The cross entropy is nothing but the mean value of the {\em surprise\/}, $ - \ln p(x_V)$. Alternatively, we can change the sign of the surprise and consider the {\em evidence\/}, $\ln p(x_V)$, leading to the 
mean value         
\begin{equation} \label{expectedevidence}
    {\rm EVIDENCE}(p) \; := \; \sum_{x_V} p^\ast(x_V) \ln p(x_V),
\end{equation}
which we also refer to as the {\em evidence\/} without explicitly highlighting the fact that it is an integrated quantity. Minimizing the KL-divergence (\ref{lo}) is then equivalent to maximizing the evidence (\ref{expectedevidence}).  
In order to be tractable, we bound the evidence from below by 
considering the set $H$ of hidden units. For any conditional probability measure $q(x_H \lvert x_V)$ and  
$p \in {\mathcal M}$, 
we then have 
\begin{eqnarray} 
{\rm EVIDENCE}(\pi_V(p))  & = & - \sum_{x_V, x_H} p^{\ast}(x_V) q(x_H | x_V) \ln \frac{q( x_H | x_V)}{p( x_V , x_H)} \nonumber \\  
    & &  
    + \sum_{x_V} p^\ast(x_V) \sum_{x_H} q(x_H | x_V) 
    \ln \frac{q(x_H | x_V)}{p(x_H | x_V)} 
          \label{kldiverg} \\  
    & \geq &  - \sum_{x_V,x_H} p^{\ast}(x_V) q(x_H | x_V) \ln \frac{q( x_H | x_V)}{p( x_V , x_H)} \label{inequ} \\  
     & =: & {\rm ELBO}(q,p).    \label{upperbound2}
\end{eqnarray}
The inequality (\ref{inequ}) follows from the non-negativity of the KL-divergences between the conditional probability distributions $q(\cdot | x_V)$ and $p(\cdot | x_V)$ in  (\ref{kldiverg}). 
The bound (\ref{upperbound2}) is referred to as the {\em evidence lower bound\/}. 
It coincides with the negative of the {\em variational free energy\/}. The importance of this quantity has been highlighted in the introduction. 
The evidence lower bound gives rise to the function 
\[
    {\rm ELBO}(q, \cdot): {\mathcal M} \to 
{\mathbb R}, \qquad p \mapsto 
    {\rm ELBO}(q,p).
\]
Replacing the evidence by the evidence lower bound implies a number of simplifications of the optimization in terms of gradient methods. One instance of these simplifications will be outlined in some more detail in Section \ref{bayesnet}. But how much do we alter the original optimization problem by this replacement? 
To get a first intuition, observe that the gap between the evidence and its lower bound is given by the mean  value (\ref{kldiverg}) of KL-divergences,  
\begin{equation} \label{gap}
    {\rm GAP}(q,p) \; := \;  
    \sum_{x_V, x_H} p^\ast(x_V) q(x_H | x_V) 
    \ln \frac{q(x_H | x_V)}{p(x_H | x_V)} .
\end{equation}
In summary, we have the following relationship between the introduced quantities:
\begin{equation*}
   {\rm EVIDENCE}(\pi_V(p)) \; = \; 
   {\rm ELBO}(q,p) + {\rm GAP}(q, p).
\end{equation*}
In Sections \ref{extendedpr} and \ref{natelbo}, we shall provide arguments supporting the hypothesis that the gap does not play a major role in learning. The main target of this article is to compare the natural gradient of ${\rm ELBO}(q, \cdot)$ on ${\mathcal M}$ with the natural gradient of 
${\rm EVIDENCE}$ or, equivalently, the objective function $D(p^\ast \| \cdot)$ on ${\mathcal M}_V$.  
In order to imply the same learning process based on the natural gradient method, the respective gradients should be consistent in a sense that we are going to specify. To reveal a condition for such a consistency, we are going to interpret the derivations of this section in a more geometric way. Before coming to this, we first review some information-geometric preliminaries.       
     
\section{Information-Geometric Preliminaries} \label{basic}    
The set ${\mathcal P}$ of strictly positive probability distributions on some 
finite set $\mathsf{X}$ of states $x$ represents the most basic example of a model within information geometry. 
We write a point $p \in {\mathcal P}$ as
\begin{eqnarray}\label{coordin}
    p \, =\, \sum_{x} p(x) \, \delta^x,
\end{eqnarray}
where $\delta^x$ denotes the Dirac measure concentrated in $x$. 
The tangent space of ${\mathcal P}$ in $p$ is given by 
\[
     T_p {\mathcal P} \; = \; \left\{ A = \sum_{x} A(x) \, \delta^x \; : \;  \sum_{x} A(x) = 0 \right\}.
\] 
For two vectors $A,B \in  T_p {\mathcal P}$, we have the {\em Fisher-Rao metric\/} 
\begin{equation} \label{FRmetric}
      g^{\rm FR}_p(A,B) \; = \; \sum_{x} 
      \frac{1}{p(x)} A(x) B(x),  
\end{equation} 
which is a Riemannian metric on ${\mathcal P}$. (Throughout this article, we also write $\langle A, B \rangle$ if there is no ambiguity regarding the Riemannian metric and the base point.)     
Furthermore, we consider the {\em Kullback-Leibler divergence\/} ({\em KL-divergence\/}) which is 
defined on ${\mathcal P} \times {\mathcal P}$ by  
\begin{equation}
   \label{KLD}
   D(q \|  p) 
   \; = \;  
   \sum_{x} q(x) \ln \frac{q(x)}{p(x)}. 
\end{equation}
Note that we already used the KL-divergence to define the primary objective function (\ref{lo}).
We can express the Fisher-Rao gradients of the KL-divergence in both arguments: 
\begin{eqnarray} 
    {\rm grad}_p D(q \| \cdot ) & = &
    \sum_{x} ( p (x) - q(x)) \, \delta^x 
    \nonumber 
    \\ 
    & = & p - q \; \in \; T_p {\mathcal P} ,\label{gradl} \\
    {\rm grad}_q D(\cdot \| p ) & = &
    \sum_{x} q(x) \left( \ln \frac{q(x)}{p(x)} - 
    \sum_{x'} q(x') \left( \ln \frac{q(x')}{p(x')} \right) \right) \delta^x \nonumber \\ 
    & = & q \left( \ln \frac{q}{p} - {\mathbb E}_q\left( \ln \frac{q}{p} \right) \right) \; \in \;  T_q {\mathcal P} . \nonumber %\label{gradm}
\end{eqnarray}  
These gradients satisfy the defining condition \eqref{eq:nat_grad_defn}, where $\mathcal{L}$ is the KL-divergence 
(\ref{KLD})
in the first and the second argument, respectively, and $g=g^{\rm FR}$ as defined by (\ref{FRmetric}). For more details, see \citep{ayamari2015, ay2017information}.
\medskip 

Now we consider the marginalization map $\pi_V: {\mathcal P} \to {\mathcal P}_V$, defined in terms of (\ref{margina}). 
In order to relate tangent vectors in 
$T_p {\mathcal P}$ to tangent vectors in 
$T_{{\pi_V}(p)} {\mathcal P}_V$, we consider the differential 
\[
    d \pi_V : T_p{\mathcal P} \to T_{{\pi_V}(p)} {\mathcal P}_V, 
\]
given by 
\begin{equation} \label{projection} 
   d\pi_V (A)(x_V) \; = \; \sum_{x_H} A(x_V,x_H). 
\end{equation} 
Furthermore, we introduce the following orthogonal spaces:
\[
    {\mathcal V}_p := {\rm ker} \, d \pi_V, \qquad  
    {\mathcal H}_p :=  {{\mathcal V}_p}^{\perp} , 
\]
where the orthogonal complement in the definition of ${\mathcal H}_p$ is meant to be with respect to the Fisher-Rao metric in $p \in {\mathcal P}$. 
We refer to ${\mathcal V}_p$ as the {\em vertical space\/} and to ${\mathcal H}_p$ as the {\em horizontal space\/} in $p$, which is in line with the differential-geometric terminology. This should not be confused with the symbols $V$ and $H$ for the visible and hidden units, respectively. In fact, by an unfortunate coincidence, the latter meaning of the symbols might even suggest the opposite naming. More precisely, the tangent space of ${\mathcal P}_V$ can be identified with the horizontal space ${\mathcal H}_p$ and not, as the symbol $V$ in ${\mathcal P}_V$ might suggest, with the vertical space.
Clearly, we have the orthogonal decomposition
\[
    T_p {\mathcal P} \; = \; {\mathcal H}_p \oplus {\mathcal V}_p \, . 
\]   
Every vector $A$ in $T_p {\mathcal P}$ has a unique representation as 
\[
     A \; = \; A^{\mathcal H} +  A^{\mathcal V} , 
\]
where $A^{\mathcal H} \in {\mathcal H}_p$ and $A^{\mathcal V} \in {\mathcal V}_p$. 
\medskip 

We now consider a model ${\mathcal M}$ in ${\mathcal P}$ and its $\pi_V$-image ${\mathcal M}_V$ and thereby restrict attention to non-singular points. A point $p \in {\mathcal M}$ is {\em admissible\/} if $p$ and $\pi_V(p)$ are non-singular points of ${\mathcal M}$ and ${\mathcal M}_V$, respectively, and  
$d\pi_V (T_p {\mathcal M}) = T_{\pi_V(p)} {\mathcal M}_V$. Admissible points allow us to locally define the geometric structures that are relevant from the perspective of information geometry. In particular,  
the model ${\mathcal M}$ carries the induced geometry of ${\mathcal P}$ in an admissible point $p$, and 
${\mathcal M}_V$ carries the corresponding induced geometry of ${\mathcal P}_V$ in $\pi_V(p)$. This will allow us to consider the gradient on ${\mathcal M}$, denoted by ${\rm grad}^{\mathcal M}$, and the gradient on ${\mathcal M}_V$, denoted by ${\rm grad}^{{\mathcal M}_V}$.   
\medskip

The objective of learning can be formulated  as the optimization of a differentiable function ${\mathcal L}: {\mathcal P}_V \to {\Bbb R}$ on ${\mathcal M}_V$ which plays the role of a primary objective function. Examples are given by the KL-divergence (\ref{lo}), which we should minimize, and 
the mean evidence (\ref{expectedevidence}), which we should maximize. In what follows, we will mainly refer to the case of minimizing ${\mathcal L}$ on ${\mathcal M}_V$ by means of the gradient descent method. Alternatively, one could also minimize the corresponding lifted function ${\mathcal L} \circ \pi_V$ defined on ${\mathcal M}$. 
More precisely, consider a curve $\gamma$ in ${\mathcal M}$ that solves the differential equation 
\begin{equation*} 
\dot{\gamma}(t) \; = \; -
{\rm grad}_{\gamma(t)}^{\mathcal M} ({\mathcal L}\circ \pi_V ) , \qquad \gamma(0) = p ,
\end{equation*}
where we assume that all points $\gamma(t)$ are admissible and $\dot{\gamma}(0) \not= 0$. (Throughout this article, we assume the existence and uniqueness of maximal solutions of differential equations without explicitly stating the conditions for this to hold.) 
Furthermore, let $\sigma := \pi_V \circ \gamma$ be the projected curve in ${\mathcal M}_V$. The change of ${\mathcal L}$ along $\sigma$ is then given by:
\[
  \frac{d}{dt} \, {\mathcal L} (\sigma(t)) \; = \; 
  \frac{d}{dt} \, {\mathcal L} (\pi_V(\gamma(t))) 
  \; = \;  
   \frac{d}{dt} \, ({\mathcal L}\circ \pi_V)(\gamma(t)) 
   \; < \; 0.
\]
This shows that the minimization of the lifted function ${\mathcal L} \circ \pi_V$ on ${\mathcal M}$ provides a useful strategy for minimizing the primary objective function ${\mathcal L}$ on ${\mathcal M}_V$. However,
even though ${\mathcal L}$ is decreasing along $\sigma$, it will typically not be following minus the gradient of ${\mathcal L}$ on ${\mathcal M}_V$.
From the chain rule we have in general that
\begin{eqnarray}
    \dot{\sigma}(t) &=& d \pi_V \left(
    \dot{\gamma}(t)
    \right) \nonumber \\
    &=& -d \pi_V \left( {\rm grad}_{\gamma(t)}^{\mathcal M} ({\mathcal L}\circ \pi_V ) \right). \label{firstfield}
\end{eqnarray}
Following the Fisher-Rao gradient on ${\mathcal M}_V$, however, would require 
\begin{equation}
    \dot{\sigma}(t) \; = \; 
    - {\rm grad}^{{\mathcal M}_V}_{\sigma(t)} {\mathcal L}. \label{secondfield} 
\end{equation}
The vector fields defined by the respective RHS of (\ref{firstfield}) and (\ref{secondfield}) are typically different, but they point, at least, in a similar direction, which is shown in the following proposition.
\medskip

\begin{proposition} \label{samedir}
Let ${\mathcal M}$ be a model in $\cP$, let ${\mathcal L}: {\mathcal P}_V \to {\mathbb R}$ be a differentiable objective function, and let $p \in {\mathcal M}$ be an admissible point. Then,
\[
   {\rm grad}^{\mathcal M}_p \left( {\mathcal L} \circ \pi_V\right) \; = \; 0
   \quad \Leftrightarrow \quad
   {\rm grad}^{{\mathcal M}_V}_{\pi_V(p)} {\mathcal L} \; = \; 0. 
\]
Furthermore, if one of the two gradients does not vanish, we have 
\[
\left\langle d\pi_V \left( {\rm grad}^{\mathcal M}_p  
\left( {\mathcal L} \circ \pi_V \right)\right), {\rm grad}^{{\mathcal M}_V}_{\pi_V(p)}  {\mathcal L} \right\rangle 
\; > \; 0.
\]
\end{proposition}
\begin{proof}
For an arbitrary $A \in T_p {\mathcal M}$, we have 
\begin{eqnarray*}
 \left\langle {\rm grad}^{\mathcal M}_p \left( {\mathcal L} \circ \pi_V\right), A \right\rangle & = & 
 d\left( {\mathcal L} \circ \pi_V \right)_p (A) \\
 & = & \left( d {\mathcal L}_{\pi_V(p)} \circ d\pi_V \right) (A) \\
 & = & d {\mathcal L}_{\pi_V(p)} \left( d\pi_V (A) \right) \\
 & = & \left\langle {\rm grad}^{{\mathcal M}_V}_{\pi_V(p)} {\mathcal L} , d \pi_V (A) \right\rangle . 
\end{eqnarray*}
This implies that ${\rm grad}^{\mathcal M}_p \left( {\mathcal L} \circ \pi_V\right)$ vanishes if and only if ${\rm grad}^{{\mathcal M}_V}_{\pi_V(p)} {\mathcal L}$ vanishes (note that the $d \pi_V(A)$, $A \in T_p {\mathcal M}$, span the tangent space $T_{\pi_V(p)} \cM_V$ because $p$ is assumed to be admissible). Furthermore, for the special case $A = {\rm grad}^{\mathcal M}_p \left( {\mathcal L} \circ \pi_V\right) \not= 0$, we obtain
\begin{eqnarray*}
\left\langle d\pi_V \left( {\rm grad}^{\mathcal M}_p  
\left( {\mathcal L} \circ \pi_V \right)\right), {\rm grad}^{{\mathcal M}_V}_{\pi_V(p)}  {\mathcal L} \right\rangle 
& = & \left\langle {\rm grad}^{\mathcal M}_p\left({\mathcal L}\circ \pi_V \right), {\rm grad}^{\mathcal M}_p\left({\mathcal L}\circ \pi_V \right)\right\rangle \\
& > & 0.
\end{eqnarray*}
\end{proof}
\medskip

We now ask the question under which conditions the two gradient fields of Proposition 
\ref{samedir} are not only pointing in a similar direction but are actually equal.  
The following definition specifies the models ${\mathcal M}$ for which this is satisfied for any objective function ${\mathcal L}: {\mathcal M} \to {\Bbb R}$, as stated in Theorem~\ref{therprev}. For such models, a projected solution curve $\sigma$ in ${\mathcal M}_V$ satisfies equation (\ref{secondfield}).  

\begin{definition}[Definition 1 of \cite{ay2020locality}] \label{defcyl}
We call a model ${\mathcal M} \subseteq {\mathcal P}$ {\em cylindrical\/} in a non-singular point $p \in {\mathcal M}$, if 
\begin{equation*}
    T_p {\mathcal M} \; = \;  (T_p {\mathcal M} \cap {\mathcal H}_p) \oplus (T_p {\mathcal M} \cap {\mathcal V}_p) .
\end{equation*}
If the model is 
cylindrical in all non-singular points $p \in {\mathcal M}$ then we call it {\em (pointwise)\/} {\em cylindrical}. 
\end{definition} 

\begin{figure}[ht]
\centering
\includegraphics[width=70mm]{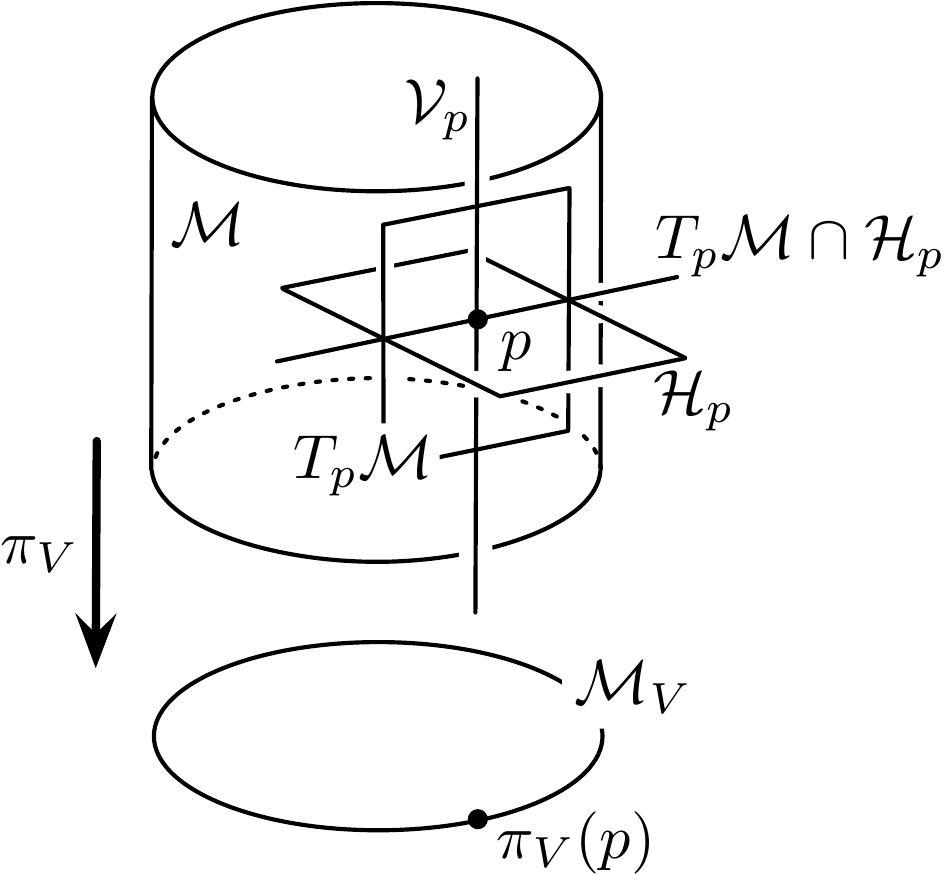}
\caption{Illustration of a cylindrical model ${\mathcal M}$ in terms of a cylinder, the Cartesian product of a circle with a finite interval. The tangent space 
$T_p {\mathcal M}$ equals the sum of its intersections with ${\mathcal H}_p$ and ${\mathcal V}_p$.}
\label{fig:CylindricalModel1}
\end{figure}             

See Figure \ref{fig:CylindricalModel1} for an illustration of a cylindrical model and Appendix \ref{sec:examples} for examples of cylindrical and non-cylindrical models.
It has been shown in \cite[Theorem 3]{ay2020locality} that a model ${\mathcal M}$ is cylindrical if and only if for the restriction $\pi_V |_{\mathcal M}: {\mathcal M} \to {\mathcal M}_V$ the following holds: Given $A,B \in \left( {\rm ker} \, d\pi_V|_{\mathcal M} \right)^\perp$, we have   
\begin{equation} \label{submersion} 
   g_p^{\rm FR}(A,B) \, = \, 
   g_{\pi_V(p)}^{\rm FR}\left( d\pi_V (A), d\pi_V(B) \right) ,
\end{equation}
whenever $p$ is {\em admissible\/}.  
The equality (\ref{submersion}) is central in the definition of a Riemannian submersion. 
The property of ${\mathcal M}$ being cylindrical ensures the invariance of the natural gradient, as stated in the following theorem. (This is different from the often stated invariance of the natural gradient under coordinate transformations, elaborated on in \citep{van2023invariance}.) 

\begin{theorem}[Theorem 5 of \cite{ay2020locality}] \label{therprev}
Let ${\mathcal M}$ be a cylindrical model, let ${\mathcal L}: {\mathcal M}_V \to {\mathbb R}$ be a differentiable objective function, and let $p \in {\mathcal M}$ be an admissible point. Then,
\begin{equation} \label{gradcyl}
   d \pi_V \left( {\rm grad}^{\mathcal M}_p ({\mathcal L}\circ \pi_V) \right) \; = \; {\rm grad}^{{\mathcal M}_V}_{\pi_V(p)} {\mathcal L}.
\end{equation}   
\end{theorem}
\medskip

Note that the gradient on the LHS of (\ref{gradcyl}) refers to the Fisher-Rao metric on ${\mathcal M} \subseteq {\mathcal P} = {\mathcal P}_{V,H}$, 
whereas the RHS refers to the Fisher-Rao metric on ${\mathcal M}_V \subseteq {\mathcal P}_V$. The invariance of the gradient as formulated in  
Theorem \ref{therprev} is quite special and holds only for the Fisher-Rao metric and cylindrical models (see \citep{ay2020locality} for further details). Our main example of a cylindrical model will be the full model ${\mathcal M} = {\mathcal P}$. 
More precisely, all points $p$ are non-singular and we obviously have $T_p {\mathcal P} = {\mathcal H}_p \oplus {\mathcal V}_p=(T_p {\mathcal P} \cap {\mathcal H}_p) \oplus (T_p {\mathcal P} \cap {\mathcal V}_p)$.
This example will provide the setting in which information-geometric quantities are studied in the absence of constraints through a lower-dimensional model. Clearly, when dealing with learning systems, we typically do have constraints. 
By relating this typical situation to the situation without constraints we are able to reveal the geometric effect of these constraints.    
\medskip 

We conclude this section with a simple statement about the orthogonal projection onto 
the tangent space of a cylindrical model.  

\begin{lemma} \label{invcylind}
Let ${\mathcal M}$ be a cylindrical model in ${\mathcal P}$, let $p$ be a non-singular point of ${\mathcal M}$, and let $\Pi_p$ denote the orthogonal projection of $T_p {\mathcal P}$ onto $T_p {\mathcal M}$. 
Then,  
\begin{equation*}
   \Pi_p ({\mathcal H}_p) 
   \; \subseteq \; {\mathcal H}_p, \qquad 
    \Pi_p ({\mathcal V}_p) 
    \; \subseteq \; {\mathcal V}_p
\end{equation*}
\end{lemma}
\begin{proof} Definition \ref{defcyl} implies the following orthogonal decomposition:
\begin{eqnarray*} 
   T_p{\mathcal P} & = & 
   T_p {\mathcal M} \oplus 
   \left(T_p {\mathcal M} \right)^\perp \\
   & = &  (T_p {\mathcal M} \cap {\mathcal H}_p) \oplus (T_p {\mathcal M} \cap {\mathcal V}_p) \oplus 
   \left(T_p {\mathcal M} \right)^\perp.
\end{eqnarray*}
Thus, every vector $X \in T_p{\mathcal P}$ has a unique orthogonal decomposition as $X = A + B + C$, where $A \in (T_p {\mathcal M} \cap {\mathcal H}_p)$, $B \in (T_p {\mathcal M} \cap {\mathcal V}_p)$, and $C \in \left(T_p {\mathcal M} \right)^\perp$. With this decomposition, we have $\Pi_p(A) = A$, $\Pi_p(B) = B$, and $\Pi_p(C) = 0$. Now, 
if $X\in {\mathcal H}_p$ then its $B$ component vanishes, so that 
$\Pi_p(X) = \Pi_p(A + C) = \Pi_p(A) + \Pi_p(C) = A \in {\mathcal H}_p$. If, on the other hand, $X \in {\mathcal V}_p$ then its $A$ component vanishes, so that $\Pi_p(X) = \Pi_p(B + C) = \Pi_p(B) + \Pi_p(C) = B \in {\mathcal V}_p$. 
\end{proof}

\section{The Extended Problem with Hidden Units} \label{extendedpr}
In this section, we are going to relate the minimization of the KL-divergence (\ref{lo}), ${\mathcal L} := D(p^\ast \| \cdot)$, on ${\mathcal M}_V$ to the minimization of the lifted function 
${\mathcal L} \circ \pi_V$ on ${\mathcal M}$. 
In general, it is difficult to minimize 
${\mathcal L}$.
In particular, we face here various challenges when trying to apply the natural gradient descent method.
On the one hand, ${\mathcal M}_V$ will typically have singularities so that gradients cannot be evaluated in these points. 
On the other hand, even for non-singular points the Fisher-Rao metric will be difficult to evaluate if we do not assume ${\mathcal M}_V$ to have a particularly simple structure. To be more concrete, we first evaluate the gradient of $D(p^\ast \| \cdot )$, considered as a function on 
${\mathcal P}_V$ (see 
equation (\ref{gradl})):
\begin{equation} \label{unconst}
     {\rm grad}^{{\mathcal P}_V}_p D(p^\ast \| \cdot ) \, = \, p - p^\ast \, \in \, T_p {\mathcal P}_V. 
\end{equation}
For the gradient on the model ${\mathcal M}_V$, we then have to project the gradient (\ref{unconst}) in $p$ onto the tangent space 
$T_p {\mathcal M}_V$, thereby assuming that $p$ is a non-singular point of ${\mathcal M}_V$. This leads to  
\begin{equation} \label{const}
     {\rm grad}^{{\mathcal M}_V}_p D(p^\ast \| \cdot ) \, = \, \Pi_p(p - p^\ast) \, \in \, T_p {\mathcal M}_V,  
\end{equation}
where $\Pi_p$ denotes the orthogonal projection onto the tangent space $T_p{\mathcal M}_V$. Note that the projected vector $\Pi_p(p - p^\ast)$ does not have to be particularly simple, even though the difference vector $p - p^\ast$, the gradient in the ambient space, is simple.   
\medskip 

We are now going to modify the problem of minimizing the KL-divergence (\ref{lo}) in several simplifying steps, thereby tracing the geometric implication of each individual step. The overall aim of this modification is to relate the minimization of (\ref{lo}), or equivalently the maximization of the evidence, to the corresponding maximization of the evidence lower bound, which will be finally addressed in Section \ref{natelbo}.  
\medskip 

It is well-known that the minimization of the KL-divergence (\ref{lo}) can be simplified by extending the problem to the 
space of probability distributions on joint states $(x_V, x_H)$ that is ${\mathcal P}_{V,H}$    
(see \cite{amari2016information}, Chapter 8). 
For that, we consider the so-called {\em data manifold\/} 
\begin{equation*} \label{datamanifold}
   {\mathcal Q} \, := \, \left\{ q \in {\mathcal P}_{V , H}  \; : \; \pi_V(q) = p^\ast \right\}.  
\end{equation*}
Note that the symbol $q$ here denotes a joint probability distribution whereas previously we have used the same symbol for the conditional probability distribution. The relation is given by $q(x_V, x_H) = p^\ast(x_V) q(x_H | x_V)$. Thus, even though it is not visible at first sight, the data manifold ${\mathcal Q}$ incorporates the data distribution $p^\ast$.
With the monotonicity of the KL-divergence, we obtain for any $p \in {\mathcal M}$ and $q \in {\mathcal Q}$ 
\begin{eqnarray*}
  ({\mathcal L} \circ \pi_{V}) (p) 
      & = & D(p^\ast  \|  \pi_V(p)) \nonumber \\ 
      & = & D(\pi_V(q) \| \pi_V(p)) \nonumber \\
      & \leq & D(q  \|  p) ,  \label{upperbound}
\end{eqnarray*}
where equality holds for 
$q = \pi_{\mathcal Q}(p)$ defined by 
\begin{equation} \label{defproj}
   \pi_{\mathcal Q}(p)(x_V, x_H)  \; = \; p^\ast(x_V) p(x_H | x_V).  
\end{equation}
Thus, we have
\begin{eqnarray}
    ({\mathcal L} \circ \pi_{V}) (p) 
       & = &  D(\pi_{\mathcal Q}(p) \| p) \nonumber  \\
        & = & \inf_{q \in {\mathcal Q}} D(q  \|  p) \nonumber \\
       & =: & D({\mathcal Q} \| p ) .  \nonumber 
\end{eqnarray}
Clearly, a point $\hat{p}$ minimizes  
${\mathcal L}\circ \pi_V = 
D( {\mathcal Q} \| \cdot )$ in  
${\mathcal M}$ if and only if 
$\pi_V(\hat{p})$ minimises 
${\mathcal L} = D(p^\ast \| \cdot)$ in 
${\mathcal M}_V$. However, there are 
important differences between the 
corresponding optimizations in terms of the natural gradient method. On the one hand, ${\mathcal M}$ typically comes with a geometric structure that simplifies the optimization of ${\mathcal L}\circ \pi_V$.  On the other hand, for the optimization of ${\mathcal L}$ 
it is natural to use the Fisher-Rao 
metric on ${\mathcal M}_V$, whereas 
${\mathcal L}\circ \pi_V$ is defined 
on ${\mathcal M}$ and should be 
optimized with respect to the 
corresponding Fisher-Rao gradient on ${\mathcal M}$.
In general, the two 
ways to optimize basically 
the same function will not be 
equivalent. 
However, according to 
Theorem \ref{therprev}, they will be 
equivalent whenever the model 
${\mathcal M}$ is cylindrical. 

\begin{theorem} \label{graddeppr}
{\bf (a)} Consider first the function $D({\mathcal Q} \| \cdot )$ on ${\mathcal P}$. Then  
\begin{equation} \label{gradientinp}
     {\rm grad}^{\mathcal P}_p  D({\mathcal Q} \| \cdot ) \; = \; p - \pi_{\mathcal Q}(p), 
\end{equation}
where $\pi_{\mathcal Q}(p)$ is defined by (\ref{defproj}).
In order to obtain the gradient of 
$D({\mathcal Q} \| \cdot )$ in a non-singular point $p \in {\mathcal M}$, we have to project (\ref{gradientinp}) onto $T_p {\mathcal M}$, that is 
\begin{equation} \label{gradientOnModel}
     {\rm grad}^{\mathcal M}_p  D({\mathcal Q} \| \cdot )  \; = \; \Pi_p (p - \pi_{\mathcal Q}(p)), 
\end{equation}
where $\Pi_p$ denotes the orthogonal projection $T_p{\mathcal P} \to T_p {\mathcal M}$ with respect to the Fisher-Rao metric on ${\mathcal P}$.
\smallskip

\noindent
{\bf (b)} If ${\mathcal M}$ is cylindrical and $p \in {\mathcal M}$ admissible then 
\begin{equation} \label{invariancedist}
  d\pi_V \left( {\rm grad}^{\mathcal M}_p  D({\mathcal Q} \| \cdot ) \right) 
  \; = \; {\rm grad}^{{\mathcal M}_V}_{\pi_V(p)}  D(p^\ast \| \cdot ).
\end{equation} 
In particular, the equality (\ref{invariancedist}) holds in all points of the maximal model ${\mathcal M} = {\mathcal P}$ where ${\mathcal M}_V = {\mathcal P}_V$. 
\end{theorem}

\begin{proof} We know that 
$D({\mathcal Q} \| \cdot) = D(p^\ast \| \pi_V(\cdot))$, 
which is a function of $p$ or, equivalently, a function of its coordinates $p(x_V, x_H)$ with respect to the basis vectors $\delta^{x_V, x_H}$ (see equation (\ref{coordin})). We evaluate the partial derivatives with respect to these coordinates,
\begin{equation*} 
   \frac{\partial}{\partial p({x}_V,x_H)} \, D(p^\ast \| \pi_V(\cdot)) \; = \; 
   - \frac{p^\ast(x_V)}{p(x_V)}, 
\end{equation*} 
and obtain for the $(x_V,x_H)$-component of the natural gradient (see \citep{ay2017information}, Proposition 2.2)
\begin{eqnarray*}
\lefteqn{
\left({\rm grad}^{\mathcal P}_p  D({\mathcal Q} \| \cdot )\right)(x_V, x_H)
   } \\
 & = & p(x_V, x_H) 
 \left( - \frac{p^\ast(x_V)}{p(x_V)} + 
 \sum_{x_V', x_H'}  p(x'_V, x'_H)  
 \frac{p^\ast(x_V)}{p(x_V)} \right)  \\
 & = & p(x_V, x_H) 
 \left( - \frac{p^\ast(x_V)}{p(x_V)} + 
 1 \right)  \\
 & = & p(x_V , x_H) - p(x_H | x_V) p^\ast (x_V) \\
 & = & p(x_V , x_H) - \pi_{\mathcal Q}(p)(x_V, x_H).
\end{eqnarray*}
This proves equation (\ref{gradientinp}), and equation (\ref{gradientOnModel}) follows immediately from that. Finally, the invariance (\ref{invariancedist}) is a direct consequence of Theorem \ref{therprev}. 
\end{proof}

\begin{figure}[ht]
\begin{center}
           \includegraphics[width=70mm]{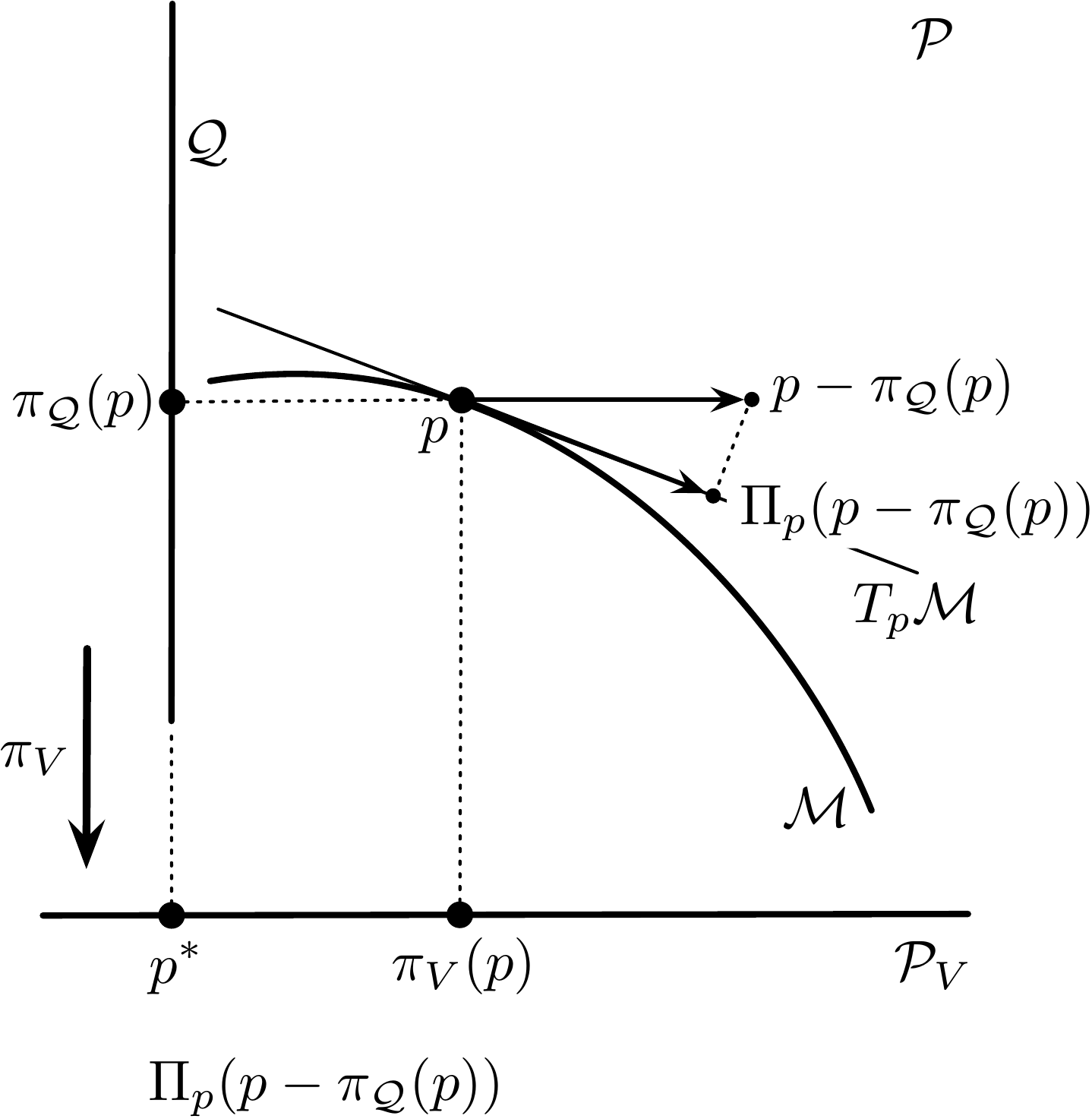}
\caption{Illustration of the gradients considered in Theorem \ref{graddeppr}.} 
\label{fig:theom_graddeppr}
\end{center}
\end{figure}   

The gradients considered in Theorem \ref{graddeppr} are graphically illustrated in Figure \ref{fig:theom_graddeppr}.
Theorem \ref{graddeppr}   
reveals a number of insights concerning the complexity and the invariance of the natural gradients which we are now going to elaborate on. First of all, it 
highlights 
the simplicity of the natural gradient of $D({\mathcal Q} \| \cdot )$ in $p \in {\mathcal P}$. It is nothing but the difference vector between $p$ and its projection $\pi_{\mathcal Q}(p)$. Thus, any complexity of the natural gradient of $D({\mathcal Q} \| \cdot )$ on a model ${\mathcal M}$ arises from the projection of that difference vector onto the tangent space $T_p {\mathcal M}$ and therefore depends very much on the structure of ${\mathcal M}$. For a Bayesian graphical model,  $T_p{\mathcal M}$ decomposes in a convenient way so that some of the original simplicity is preserved after projection. A corresponding more precise statement will be formulated at the end of this article, in Proposition \ref{orthogonalVectors}. 
Furthermore, the gradient (\ref{gradientOnModel}) of the function $D({\mathcal Q} \|\cdot)$, defined on ${\mathcal M}$, can now be compared with the gradient (\ref{const}) of  the original function $D(p^\ast \| \cdot)$ which is defined on ${\mathcal M}_V$. According to the invariance (\ref{invariancedist}), 
these two gradients are equivalent, if ${\mathcal M}$ is cylindrical, which implies that gradient descent learning in ${\mathcal M}$ yields exactly the same trajectories as the gradient descent learning in ${\mathcal M}_V$. This is a consequence of the corresponding invariance of the Fisher-Rao metric as formulated by Chentsov and not at all given for other choices of Riemannian  metrics \citep{chentsov82}. While the requirement for a model to be cylindrical is quite restrictive, it holds for the full model ${\mathcal M} = {\mathcal P}$. This brings us to the last insight of Theorem \ref{graddeppr}. If we do not restrict the optimization to a lower-dimensional model ${\mathcal M}$ then all information-geometric structures are consistent in the sense that the optimization in the extended system, with hidden units, is equivalent to the original optimization with only visible units. Again, any deviation from the invariance (\ref{invariancedist}) arises from the restriction of the optimization to ${\mathcal M}$. 
\medskip

We can illustrate the invariance \eqref{invariancedist} and a possible deviation from it using two example models, denoted by ${\mathcal M}^{(a)}$ and ${\mathcal M}^{(b)}$, which we introduce in what follows.
(The code for reproducing the data and figures in this paper is made available at \cite{code_for_paper}.)
Consider three binary random variables $X_s, X_{t_1}, X_{t_2}$. The manifold  $\cP$ consists of all joint probability distributions $p(x_s, x_{t_1}, x_{t_2})$. With $H = \{s\}$ and $V = \{t_1, t_2\}$, let $\pi_V$ be the marginalization map over $X_s$, i.e.
\begin{eqnarray*}
        \pi_V\colon \cP &\to& \cP_V\\
        p(x_s,x_{t_1},x_{t_2}) &\mapsto& p(x_{t_1},x_{t_2}) = \sum_{x_s} p(x_s,x_{t_1},x_{t_2}). 
\end{eqnarray*}
For a general model $\cM \subseteq \cP$ we study the following curves.  We let $\sigma_0 \in \cM_V$ be the integral curve of the gradient of the primary objective function $D(p^\ast \| \cdot)$, i.e. solving the differential equation 
\begin{equation} \label{eq:traj2}
    \dot\sigma_0(t) \, = \, -{\rm grad}_{\sigma_0(t)}^{{\mathcal M}_V} D(p^\ast \| \cdot), \ \ \ \sigma_0(0) \, =\,  \pi_V(p). 
\end{equation}
Furthermore, we let the curve $\gamma_1$ in $\cM$ be the solution of the differential equation 
\begin{equation} \label{eq:traj1}
    \dot\gamma_1(t) \, =\,  -{\rm grad}^{{\mathcal M}}_{\gamma_1(t)} D(\cQ \| \cdot ), \ \ \ \gamma_1(0) \,=\, p, 
\end{equation}
and $\sigma_1 = \pi_V \circ \gamma_1$ in $\cM_V$ be the projection of $\gamma_1$.  
\medskip

\begin{figure}[h]
    \centering
    \begin{tikzpicture}
        \node at (0,0) {
        \begin{tikzpicture}[neuron/.style={circle,draw, minimum size=.7cm, inner sep=0}]
            \node (g) at (-1.3,.8) {$G^{(a)}$};
            \node (a) [neuron] at (0,0) {$s$};
            \node (b) [neuron] at (-1,-.7) {$t_1$};
            \node (c) [neuron] at (1,-.7) {$t_2$};
        \end{tikzpicture}
        };
        \node at (5,0) {
        \begin{tikzpicture}[neuron/.style={circle,draw, minimum size=.7cm, inner sep=0}]
            \node (g) at (-1.3,.8) {$G^{(b)}$};
            \node (a) [neuron] at (0,0) {$s$};
            \node (b) [neuron] at (-1,-.7) {$t_1$};
            \node (c) [neuron] at (1,-.7) {$t_2$};
            \draw [->] (a) -- (b);
            \draw [->] (a) -- (c);
        \end{tikzpicture}};
    \end{tikzpicture}
    \caption{Graphical representations of the models $\cM^{(a)}$ and $\cM^{(b)}$.}
    \label{fig:3NodeModels}
\end{figure}

We are now going to define two models in $\cP$, ${\mathcal M}^{(a)}$ and
${\mathcal M}^{(b)}$, given by the corresponding graphs $G^{(a)}$ and $G^{(b)}$ in Figure \ref{fig:3NodeModels}.
We begin with the model $\cM^{(a)}$, which we define as the set of probability distributions for which $X_s, X_{t_1}, X_{t_2}$ are independent, that is,
\begin{equation}\label{eq:cylindrical_model_def}
    \cM^{(a)} \,=\, \{p \in \cP \,:\, p(x_s, x_{t_1}, x_{t_2}) \,=\, p(x_s)p(x_{t_1})p(x_{t_2}) \}. 
\end{equation}
Graphically, these are all the distributions factorizing over the graph $G^{(a)}$ in Figure \ref{fig:3NodeModels}. It can be shown that this model is cylindrical.\footnote{See Example 1 in Appendix \ref{sec:examples} for a two-node example of this.}
Owing to Theorem \ref{graddeppr}, we know that the curves $\sigma_0$ and $\sigma_1$ are identical.
This is shown in Figure \ref{fig:cylindrical_with_indep_manifold_ELBO}, where the model $\cM^{(a)}_V$ is plotted by the blue grid and the indistinguishable curves $\sigma_0$ and $\sigma_1$ are denoted by the solid black line.
Note that the black line also represents the gradient curve coming from the evidence lower bound which is going to be discussed in the next section.

\medskip

Similarly, we now define
\begin{equation}\label{eq:non_cylindrical_model_def}
    \cM^{(b)} \,=\, \{ p \in \cP \,:\, p(x_s, x_{t_1}, x_{t_2}) \,=\, p(x_s)p(x_{t_1}|x_s)p(x_{t_2}|x_s) \}.
\end{equation}
It consists of those distributions that factorize over the graph $G^{(b)}$ in Figure \ref{fig:3NodeModels}. In Example 3 of Appendix \ref{sec:examples} we show that this model is not cylindrical. The model $\cM^{(b)}_V$ is in this case the full simplex $\mathcal{P}_V$.  
Figure \ref{fig:KL_traj_2_all_three_curves} (top) shows the trajectories of the curves $\sigma_0$ and $\sigma_1$ defined by \eqref{eq:traj2} and \eqref{eq:traj1}, respectively, using dashed blue and solid green lines. (The solid red lines are related to the ELBO objective function elaborated on in the next section.) 
Figure \ref{fig:KL_traj_2_all_three_curves} (bottom-left) shows the same trajectories in coordinates as functions of time and Figure \ref{fig:KL_traj_2_all_three_curves} (bottom-right) shows the KL-divergence evaluated on these trajectories as a function of time.
Note that now the trajectories of $\sigma_0$ and $\sigma_1$ do not coincide, deviating from the situation of Figure \ref{fig:cylindrical_with_indep_manifold_ELBO}, due to $\cM^{(b)}$ not being cylindrical. 
In spite of this, the trajectories converge to the target distribution as evident from the top two rows.
Furthermore, the evaluations of the KL-divergence along these different trajectories is almost indistinguishable.
Since this decay of KL-divergence corresponds directly to the speed of learning, understanding the effect of a model being cylindrical on the speed of learning is an important question for future research.
Finally, observe that since $\mathcal{M}_V^{(b)} = \mathcal{P}_V$, the trajectories of the integral curves $\sigma_0$ (dashed blue) of the gradient of $D(p^\ast \| \cdot)$ are straight lines. This is a consequence of \eqref{gradl} and has more general implications on the learning  \citep{datar2025convergence}.
\medskip 

In order to measure the deviation from the invariance \eqref{invariancedist} we evaluate the cosine similarity of the involved vector fields. More precisely, we compute the cosine similarity between the vectors ${\rm grad}^{{\mathcal M}_V}_{\pi_V(p)}  D(p^\ast \| \cdot )$ and $d\pi_V \left( {\rm grad}^{\mathcal M}_p  D({\mathcal Q} \| \cdot ) \right)$ in $T_p\cM_V^{(b)}$, thereby assuming them to be non-zero.  
In general, the cosine similarity between two non-zero vectors $A$ and $B$ in an inner product space is defined as
\begin{equation} \label{eq:cosinesim}
    \cos(\alpha) \,=\, \frac{\langle A, B\rangle}{\|A\| \|B\|}, 
\end{equation}
where $\alpha$ is the angle between $A$ and $B$. The cosine similarity reaches its maximal value $1$ when $A$ and $B$ point in the same direction and its minimal value $-1$ when $A$ and $B$ point in opposite directions. Clearly, in our setting the inner product and the norm are given by the Fisher-Rao metric in $T_p\cM^{(b)}_V$.
For a fixed $p^*$, the cosine similarity depends on the base point $p$ which is sampled from $\cM^{(b)}$ according to Jeffrey's prior. 
The results are plotted in the histograms in Figure \ref{fig:hist_comparison} (left). 
As one can see, most points are close to $1$. 
In fact, more than $83\%$ of the samples are above $0.7$. 
Furthermore, all of the values are larger than zero, which means that both vector fields qualify for optimizing the primary objective function $D(p^* \| \cdot )$.
This is not a coincidence and follows directly from Proposition \ref{samedir}.
\medskip

\section{The Natural Gradient of the Evidence Lower Bound} \label{natelbo}
In Theorem \ref{graddeppr}, we have related the gradient of the primary objective functions $D(p^\ast \| \cdot)$ to the gradient of $D({\mathcal Q} \| \cdot)$.  
We are now going to replace the entire set ${\mathcal Q}$ in $D({\mathcal Q} \| \cdot)$ by a single point $q \in {\mathcal Q}$, leading to the upper bound $D(q \| \cdot)$. In Theorem \ref{mainthdist} below, we will then  relate the gradient of the primary objective function $D(p^\ast \| \cdot)$ to the gradient of $D(q \| \cdot)$. This will finally allow us to study the natural gradient of the evidence lower bound in relation to the natural gradient of the evidence.
\medskip

For any 
$q \in {\mathcal Q}$ and $p \in {\mathcal M}$, we apply the Pythagorian relation and obtain  
\begin{eqnarray}
   {D({\mathcal Q} \| p)} 
    & \leq & D({\mathcal Q} \| p) + {D(q \| \pi_{\mathcal Q}(p) ) } \label{addterm}\\
   & = & {D(\pi_{\mathcal Q}(p) \| p) } + {D(q \| \pi_{\mathcal Q}(p) ) } .   
   \label{firstand second} \nonumber \\
       & = & D(q \|  p)  \label{upperbound3} 
\end{eqnarray}
It can be easily verified that the additional term
in  (\ref{addterm}), $D(q \| \pi_{\mathcal Q}(p))$,   coincides with ${\rm GAP}(q,p)$ as defined by (\ref{gap}). 
Now, instead of taking the gradient of ${D({\mathcal Q} \| \cdot)}$
we take the gradient of the upper bound (\ref{upperbound3}), with a fixed $q$, and analyze the effect of this replacement. For that, let us first rewrite this upper bound as follows:
\begin{equation} \label{twocompo}
   D(q \| p) \; = \; 
   D({\mathcal Q} \| p ) + {\rm GAP}(q,p).
\end{equation}
For the gradient of $D(q \| \cdot )$ in ${\mathcal P}$, we obtain  
\begin{eqnarray}
         {\rm grad}^{\mathcal P}_p D(q \| \cdot) &=& {\rm grad}^{\mathcal P}_p \, D({\mathcal Q} \| \cdot ) + 
{\rm grad}^{\mathcal P}_p \, {\rm GAP}(q,\cdot) \nonumber \\
&=& (p - \pi_{\mathcal Q}(p)) + 
   (\pi_{\mathcal Q}(p) - q) .
   \qquad\quad  \mbox{(by (\ref{gradientinp}) and 
   (\ref{gradl}))}
   \label{newcomp} 
\end{eqnarray}
It is easy to see that the first difference vector in (\ref{newcomp}), {$p - \pi_{\mathcal Q}(p)$}, is an element of the horizontal space ${\mathcal H}_p$, whereas the second one, $\pi_{\mathcal Q}(p) - q$, is an element of the vertical space ${\mathcal V}_p$. Therefore, the $d\pi_V$-image of the latter difference vector vanishes:  
\begin{eqnarray*}
d\pi_V (\pi_{\mathcal Q}(p) - q) 
   & = & \sum_{x_H} 
     \left( p^\ast(x_V) p(x_H | x_V) - q(x_V, x_H) \right) \qquad (\mbox{by (\ref{projection})})  \\
  & = & p^\ast(x_V) \sum_{x_H} 
     \left( p(x_H | x_V) - q(x_H | x_V) \right)   \\
  & = & p^\ast(x_V) (1 - 1) \; = \; 0, 
\end{eqnarray*}
or equivalently,
\[
   d\pi_V \left( {\rm grad}^{\mathcal P}_p \, {\rm GAP}(q, \cdot)\right) \; = \; 0.  
\] 
In summary, we obtain 
\begin{eqnarray*}
    d\pi_V \left( {\rm grad}^{\mathcal P}_p D(q \| \cdot)\right)  
    & = & d\pi_V \left( p - \pi_{\mathcal Q}(p) \right) \\
    & = & \pi_V(p) - p^{\ast} \\
    & = & {\rm grad}_{\pi_V(p)}^{{\mathcal P}_V} D(p^{\ast} \| \cdot) . 
\end{eqnarray*}
This shows that the Fisher-Rao gradient of the primary objective function $D(p^{\ast} \| \cdot)$ on ${\mathcal P}_V$ is not affected at all by the extension of the problem to the set ${\mathcal P} = {\mathcal P}_{V,H}$. When we replace ${\mathcal P}$ by a more general model 
${\mathcal M}$ this invariance only holds, if ${\mathcal M}$ is cylindrical.

\begin{theorem} \label{mainthdist}
Let ${\mathcal M}$ be a cylindrical model in ${\mathcal P}$, let $p \in {\mathcal M}$ be admissible, and let $q \in {\mathcal Q}$. Then
\begin{equation}  
 d\pi_V \left( {\rm grad}^{{\mathcal M}}_p \, 
   {\rm GAP}(q, \cdot ) \right) 
   \; = \; 0, \label{vanishgap} 
\end{equation}
and therefore 
\begin{equation} 
   d\pi_V \left( {\rm grad}^{{\mathcal M}}_p 
   D(q \| \cdot ) \right) 
           \; = \;  {\rm grad}_{\pi_V(p)}^{{\mathcal M}_V} \, D(p^\ast \| \cdot). \label{extension}
\end{equation}  
\end{theorem} 
\begin{proof} 
We begin with the gradient of ${\rm GAP}(q,\cdot)$:
\begin{equation*}
  {\rm grad}^{{\mathcal M}}_p \, {\rm GAP}(q, \cdot ) 
        \; = \; \Pi_p \left( {\rm grad}^{{\mathcal P}}_p \,
        {\rm GAP}(q,\cdot ) \right) 
        \; = \; \Pi_p \left( 
        \pi_{\mathcal Q}(p) - q
        \right) 
        \; = \;  \Pi_p \left( {(p - q)}^{\mathcal V} \right). 
\end{equation*}
We know, by definition, that 
${(p - q)}^{\mathcal V}$ is contained in ${\mathcal V}_p$. 
According to Lemma \ref{invcylind}, the vector ${(p - q)}^{\mathcal V}$ remains in  ${\mathcal V}_p$ after projecting it onto the tangent space $T_p {\mathcal M}$, that is,  
\[
   \Pi_p \left( {(p - q)}^{\mathcal V} \right) \; \in \; {\mathcal V}_p. 
\]
Therefore, this vector is mapped via $d\pi_V$ to $0$, which proves 
(\ref{vanishgap}). 
With this, we have   
\begin{eqnarray*}
  d\pi_V \left( {\rm grad}^{{\mathcal M}}_p \, D(q \| \cdot ) \right)  
        & = & 
   d\pi_V \left( {\rm grad}^{{\mathcal M}}_p \, 
   D({\mathcal Q} \| \cdot ) \right) +   
   d\pi_V \left( {\rm grad}^{{\mathcal M}}_p \, 
   {\rm GAP}(q, \cdot )\right) \\
   & = & d\pi_V \left( {\rm grad}^{{\mathcal M}}_p \, 
   D({\mathcal Q} \| \cdot ) \right), 
\end{eqnarray*}
and with (\ref{invariancedist}) of Theorem  \ref{graddeppr} 
we finally obtain (\ref{extension}).
\end{proof}

\begin{figure}[ht]
\begin{center}           \includegraphics[width=70mm]{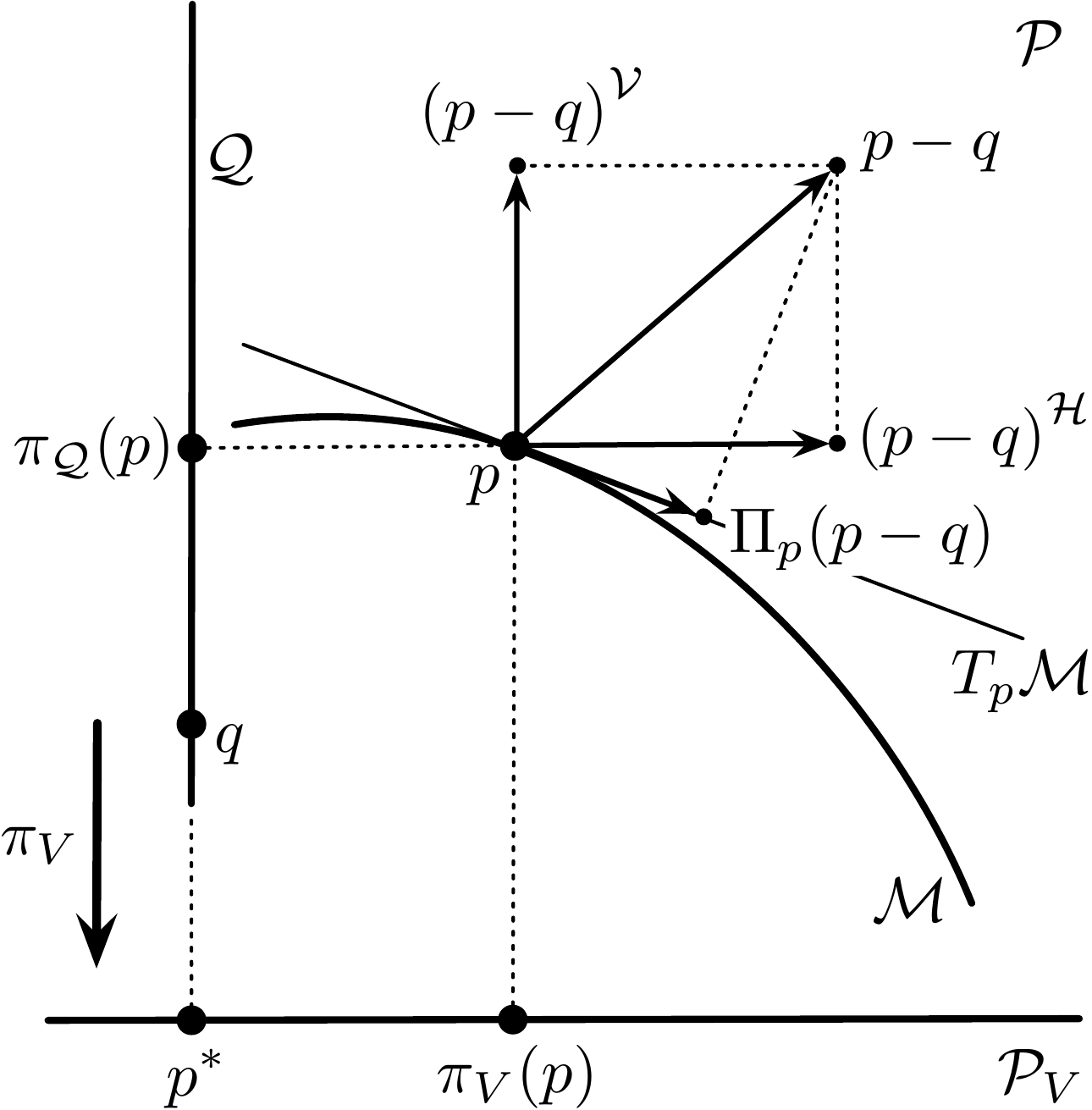}
\caption{Illustration of gradients considered in Theorem \ref{mainthdist}.}
\label{fig:theom_mainthdist}
\end{center}
\end{figure}   
The gradients considered in Theorem \ref{mainthdist} are graphically illustrated in Figure \ref{fig:theom_mainthdist}.
Note that while the invariance (\ref{extension}) appears to be very similar to the invariance (\ref{invariancedist}),  
it is in fact quite different. The main difference is that the objective function on ${\mathcal M}$, the function $D(q \| \cdot)$, is not ``just'' the pull-back of an objective function on ${\mathcal M}_V$. It consists of a pull-back component, the first term on the RHS of (\ref{twocompo}), and the function ${\rm GAP}(q,\cdot)$, the second term on the RHS of (\ref{twocompo}), which varies only in vertical direction so that the $d\pi_V$-image of its gradient vanishes. 
Note that this result, expressed by (\ref{vanishgap}), depends crucially on the information-geometric structures and does not hold for the standard Euclidean geometry defined in terms of a coordinate system. We give an example in Appendix \ref{sec:euclid}.
\medskip

We illustrate the invariance \eqref{extension} and a possible deviation from it by revisiting the example models $\cM^{(a)}$ and $\cM^{(b)}$ from the previous section depicted in Figure \ref{fig:3NodeModels}. Now, for a general model $\cM$ we let the curve $\gamma_2$ in $\cM$ be the solution to the differential equation 
\begin{equation*} 
    \dot\gamma_2(t) \,=\, - {\rm grad}^{{\mathcal M}}_{\gamma_2(t)} \, D(q \| \cdot ), \ \ \ \gamma_2(0) \,=\, p, 
\end{equation*}
and $\sigma_2 = \pi_V \circ \gamma_2$ the projection of $\gamma_2$.
Since $\cM^{(a)}$ is cylindrical, we can use Theorems \ref{graddeppr} and \ref{mainthdist}
to conclude that the curves $\sigma_0$, $\sigma_1$ and $\sigma_2$ are identical.
This is depicted in Figure \ref{fig:cylindrical_with_indep_manifold_ELBO} where the blue grid represents the cylindrical model $\cM^{(a)}_V$, as a subset of $\cP_V$, and the solid black line represents the indistinguishable curves $\sigma_0$, $\sigma_1$ and $\sigma_2$.
\medskip

\label{discussion-Mb-2}
For the non-cylindrical model $\cM^{(b)}$ we plot the trajectory of $\sigma_2$ in Figure \ref{fig:KL_traj_2_all_three_curves} (top) in solid red and compare it with the trajectory of $\sigma_0$ (shown in dashed blue) and $\sigma_1$ (shown in solid green).
Figure \ref{fig:KL_traj_2_all_three_curves} (bottom-left) shows the trajectories in coordinates as functions of time and Figure \ref{fig:KL_traj_2_all_three_curves} (bottom-right) shows the KL-divergence evaluated on these trajectories as a function of time.
Note again that the trajectory $\sigma_2$ is distinct from $\sigma_0$ and $\sigma_1$.
However, $\sigma_2$ also converges to the same target distribution as that of $\sigma_0$ or $\sigma_1$ and the KL-divergence evaluation is barely distinguishable.
% Figure for the cylinderical model
\begin{figure}[h!]
    \centering	  \input{data_for_plots/cylindrical_with_indep_manifold_new_color}
	\centering
 \vspace{-1cm}
	\caption{The blue grid represents the set of independent probability distributions over two random variables corresponding to the cylindrical model $\cM^{(a)}_V$, as a subset of $\cP_V$.
    The cylindrical model $\cM^{(a)}$ is defined in \eqref{eq:cylindrical_model_def} and depicted in Figure \ref{fig:3NodeModels}.
    The solid black line shows the overlapping trajectories $\sigma_0$, $\sigma_1$ and $\sigma_2$, where $\sigma_0$ satisfies $ \dot\sigma_0(t) = -{\rm grad}_{\sigma_0(t)}^{{\mathcal M}^{(a)}_V} D(p^\ast \| \cdot)$, $\sigma_1=\pi_V \circ \gamma_1$ is the projection of the negative gradient curve $\gamma_1$  satisfying 
    $ \dot\gamma_1(t) = -{\rm grad}^{{\mathcal M}^{(a)}}_{\gamma_1(t)} \, D(\cQ \| \cdot )$ and $\sigma_2=\pi_V \circ \gamma_2$ is the projection of the negative gradient curve $\gamma_2$ satisfying $  \dot\gamma_2(t) = - {\rm grad}^{{\mathcal M}^{(a)}}_{\gamma_2(t)} \, D(q \| \cdot)$.
     Theorem \ref{graddeppr} and \ref{mainthdist} imply that $\sigma_0$, $\sigma_1$ and $\sigma_2$ are identical.
    }
\label{fig:cylindrical_with_indep_manifold_ELBO}
\end{figure}
In line with the discussion on cosine similarity from the last section, we evaluate the cosine similarity between the vectors ${\rm grad}^{{\mathcal M}_V}_{\pi_V(p)}  D(p^\ast \| \cdot )$ and $d\pi_V \left( {\rm grad}^{\mathcal M}_p D(q \| \cdot ) \right)$ defined in equation \eqref{eq:cosinesim}. 
We sample $p$ in $\cM^{(b)}$ according to Jeffrey's prior and plot the results in the histograms in Figure \ref{fig:hist_comparison} (right).
As one can see, the points are still close to 1 but on average lower than the cosine similarities between ${\rm grad}^{{\mathcal M}_V}_{\pi_V(p)}  D(p^\ast \| \cdot )$ and $ d \pi_V \left( {\rm grad}_p^{{\mathcal M}} \, 
D({\mathcal Q} \| \cdot ) \right)$. Moreover, note that in this case not all points lie above zero anymore.
Investigating conditions under which this cosine similarity is negative is an interesting question for future research. 
\begin{figure}
	\centering
	\input{data_for_plots/simplex_non_cylindrical_1_three_curves}
    \input{data_for_plots/coordinates_traj_1_three_curves}	
	% This file was created by matlab2tikz.
%
%The latest updates can be retrieved from
%  http://www.mathworks.com/matlabcentral/fileexchange/22022-matlab2tikz-matlab2tikz
%where you can also make suggestions and rate matlab2tikz.
%
\begin{tikzpicture}

\begin{axis}[%
width=0.43*4.521in,
height=0.43*3.566in,
at={(0.758in,0.481in)},
scale only axis,
xmin=0,
xmax=40,
xlabel style={font=\color{white!15!black}},
xlabel={time ($t$)},
ymin=0,
ymax=1.2,
ylabel style={font=\color{white!15!black}},
ylabel={$D(p^{\ast}||p(t))$},
axis background/.style={fill=white},
title style={font=\bfseries},
title={KL-divergence on $\sigma_0$, $\sigma_1$, $\sigma_2$}
]
\addplot [color=red]
  table[row sep=crcr]{%
1	1.13174732449755\\
2	0.881592610824368\\
3	0.720277057620936\\
4	0.601842288530546\\
5	0.509534252723522\\
6	0.435061784869224\\
7	0.373623486932604\\
8	0.322151341162223\\
9	0.278548163800667\\
10	0.241308220609004\\
11	0.209309140297655\\
12	0.181689096695929\\
13	0.157770065828528\\
14	0.13700765168817\\
15	0.118957066251296\\
16	0.103249374822187\\
17	0.0895745124294371\\
18	0.0776689098538578\\
19	0.067306341842598\\
20	0.0582910759673751\\
21	0.0504526899973381\\
22	0.043642110793444\\
23	0.0377285498685967\\
24	0.0325970943358201\\
25	0.0281467718132586\\
26	0.0242889529924399\\
27	0.0209459913116144\\
28	0.0180500283550427\\
29	0.0155419174851606\\
30	0.0133702371902407\\
31	0.0114903797680531\\
32	0.00986371043660796\\
33	0.00845679724430185\\
34	0.00724071404502115\\
35	0.00619041831054236\\
36	0.00528420369432078\\
37	0.00450322491696564\\
38	0.0038310903541578\\
39	0.00325351604873771\\
40	0.00275803388632757\\
41	0.00233374634791098\\
42	0.00197112046563268\\
43	0.00166181420271012\\
44	0.00139852929882061\\
45	0.00117488553659222\\
46	0.000985312293890274\\
47	0.000824954084681947\\
48	0.000689587521180606\\
49	0.000575547736595914\\
50	0.000479662791544141\\
};
\addplot [color=blue, dashed]
  table[row sep=crcr]{%
1	1.13174732449755\\
2	0.837525440979601\\
3	0.670999537372635\\
4	0.55525335255496\\
5	0.467639196541227\\
6	0.398214461708968\\
7	0.34162713725406\\
8	0.294619963449455\\
9	0.255041173418504\\
10	0.221383793372278\\
11	0.192545485254865\\
12	0.167692417777385\\
13	0.146176936337598\\
14	0.127485128557905\\
15	0.111201991077076\\
16	0.0969874637005775\\
17	0.0845594457242519\\
18	0.0736814525997729\\
19	0.064153447709751\\
20	0.0558049024732726\\
21	0.0484894555664781\\
22	0.0420807426028575\\
23	0.0364690977432941\\
24	0.0315589151742417\\
25	0.0272665170882882\\
26	0.0235184154123559\\
27	0.0202498831116536\\
28	0.0174037713255591\\
29	0.0149295233956011\\
30	0.0127823477148741\\
31	0.0109225194146111\\
32	0.00931478700835229\\
33	0.00792786480149611\\
34	0.00673399554232332\\
35	0.00570857071831264\\
36	0.00482979827998385\\
37	0.00407840953278033\\
38	0.00343739856221207\\
39	0.0028917889073865\\
40	0.00242842331453846\\
41	0.00203577331618506\\
42	0.00170376611807243\\
43	0.00142362685696829\\
44	0.00118773473788332\\
45	0.000989491889120717\\
46	0.000823204006926981\\
47	0.000683972017170617\\
48	0.000567594077172929\\
49	0.000470477292812159\\
50	0.000389558548591549\\
};
\addplot [color=olivegreen]
  table[row sep=crcr]{%
1	1.13174732449755\\
2	0.888213151271917\\
3	0.719818350412611\\
4	0.595093736767402\\
5	0.498576738944578\\
6	0.421630964266045\\
7	0.358977077393503\\
8	0.307161447268425\\
9	0.263803651716861\\
10	0.227194254451613\\
11	0.196064693942496\\
12	0.169448096852925\\
13	0.146591103485658\\
14	0.126895795494172\\
15	0.109880165376304\\
16	0.0951504350711913\\
17	0.0823811920118997\\
18	0.0713008284484672\\
19	0.0616806677456967\\
20	0.0533267101991711\\
21	0.0460732761995211\\
22	0.0397780474580934\\
23	0.0343181542384075\\
24	0.029587055847932\\
25	0.0254920299212354\\
26	0.0219521338014989\\
27	0.0188965353167743\\
28	0.0162631348151674\\
29	0.0139974183574609\\
30	0.0120514954071065\\
31	0.0103832845267118\\
32	0.00895581839514767\\
33	0.00773664552932949\\
34	0.00669731086811724\\
35	0.00581290116021194\\
36	0.00506164411112774\\
37	0.00442455264291633\\
38	0.00388510752004846\\
39	0.00342897308456923\\
40	0.0030437419947818\\
41	0.00271870573432663\\
42	0.00244464830479813\\
43	0.00221366098109795\\
44	0.00201897633527561\\
45	0.00185481995709119\\
46	0.00171627844798777\\
47	0.00159918236419574\\
48	0.00150000285363336\\
49	0.00141576078452208\\
50	0.00134394721126099\\
};
\end{axis}
\end{tikzpicture}%
	\caption{These figures study different gradient trajectories on the model $\cM = \cM^{(b)}$ defined in 
\eqref{eq:non_cylindrical_model_def} and discussed on pages \pageref{eq:non_cylindrical_model_def} 
    and \pageref{discussion-Mb-2}. The
    top figure shows  the curve $\sigma_0$  satisfying $ \dot\sigma_0(t) = -{\rm grad}_{\sigma_0(t)}^{{\mathcal M}_V} D(p^\ast \| \cdot)$ (dashed-blue),  the curve $\sigma_1=\pi_V \circ \gamma_1$, with $\gamma_1$ satisfying $ \dot\gamma_1(t) = -{\rm grad}^{{\mathcal M}}_{\gamma_1(t)} \, D(\cQ \| \cdot )$ (solid green) and the curve $\sigma_2=\pi_V \circ \gamma_2$, with $\gamma_2$ satisfying $  \dot\gamma_2(t) = - {\rm grad}^{{\mathcal M}}_{\gamma_2(t)} \, D(q \| \cdot)$ (solid red) for a fixed target distribution (cross).
   The bottom-left figure shows the trajectories $\sigma_0$, $\sigma_1$ and $\sigma_2$ in coordinates as functions of time. The bottom-right figure shows the KL-divergence evaluated on the trajectories $\sigma_0$, $\sigma_1$ and $\sigma_2$.}      
	\label{fig:KL_traj_2_all_three_curves}
\end{figure}
According to equation \eqref{extension}, minimizing the primary objective function $D(p^\ast \| \cdot)$ on ${\mathcal M}_V$ is equivalent to minimizing the function $D(q \| \cdot)$ on ${\mathcal M}$ whenever ${\mathcal M}$ is cylindrical. We now show that this is, at the same time, equivalent to maximizing the evidence lower bound. For any $q \in {\mathcal Q}$ and $p \in {\mathcal M}$, we have   
\begin{eqnarray*}
   D(q \|  p)   
        & = &  \sum_{{x}_V , {x}_H} p^\ast({x}_V) \, q({x}_H | {x}_V ) 
               \ln \frac{p^\ast({x}_V) \, q( {x}_H | {x}_V )}{p( {x}_V , {x}_H )}  \\
        & = & \underbrace{\sum_{x_V} p^\ast({x}_V) \ln p^\ast({x}_V)}_{\leq \, 0} +  
      \sum_{{x}_V , {x}_H} p^\ast({x}_V) \, q({x}_H | {x}_V ) 
               \ln \frac{q( {x}_H | {x}_V )}{p( {x}_V , {x}_H)}  \\
        & = & - {\rm ELBO}(q, p) + \mbox{const.},  
\end{eqnarray*}
where ${\rm ELBO}(q,p)$ is defined by  (\ref{upperbound2}). Thus, the gradient of the function $D(q \| \cdot)$ will be the same as the gradient 
of $- {\rm ELBO}(q, \cdot)$, because the two functions differ only by a constant. More precisely, we have for all non-singular points $p$ of ${\mathcal M}$   
\begin{equation*} \label{elbokl}
   {\rm grad}^{{\mathcal M}}_p \, {\rm ELBO}(q , \cdot ) 
        \; = \; - {\rm grad}^{{\mathcal M}}_p D(q \| \cdot ). 
\end{equation*} 
As we know that the minimization of $D(p^\ast \| \cdot)$ is equivalent to the maximization of the evidence, we have the following immediate consequence of Theorem \ref{mainthdist}.
\begin{figure}
	\centering
    % This file was created by matlab2tikz.
%
%The latest updates can be retrieved from
%  http://www.mathworks.com/matlabcentral/fileexchange/22022-matlab2tikz-matlab2tikz
%where you can also make suggestions and rate matlab2tikz.
%
\definecolor{mycolor1}{rgb}{0.00000,0.44700,0.74100}%
\begin{tikzpicture}

\begin{axis}[%
width=0.65*3.521in,
height=0.65*2.566in,
at={(0.758in,0.481in)},
scale only axis,
xmin=-1,
xmax=1,
ymin=0,
ymax=0.85,
ytick={0.2,0.4,0.6,0.8},
ylabel style={font=\color{white!15!black}},
ylabel={Fraction of samples},
xlabel={Cosine similarity},
axis background/.style={fill=white},
title = { $ d \pi_V \left( {\rm grad}^{{\mathcal M}} \, D({\mathcal Q} \| \cdot ) \right)$}
]
\addplot[ybar interval, fill=mycolor1, fill opacity=0.6, draw=black, area legend] table[row sep=crcr] {%
x	y\\
0	2e-05\\
0.1	0.00036\\
0.2	0.002\\
0.3	0.00774\\
0.4	0.01966\\
0.5	0.03646\\
0.6	0.07816\\
0.7	0.12698\\
0.8	0.19292\\
0.9	0.5357\\
1	0.5357\\
};
\end{axis}
\end{tikzpicture}%
    % This file was created by matlab2tikz.
%
%The latest updates can be retrieved from
%  http://www.mathworks.com/matlabcentral/fileexchange/22022-matlab2tikz-matlab2tikz
%where you can also make suggestions and rate matlab2tikz.
%
\definecolor{mycolor1}{rgb}{0.00000,0.44700,0.74100}%
\begin{tikzpicture}

\begin{axis}[%
width=0.65*3.521in,
height=0.65*2.566in,
at={(0.758in,0.481in)},
scale only axis,
xmin=-1,
xmax=1,
ymin=0,
ymax=0.85,
ytick={0.2,0.4,0.6,0.8},
ylabel style={font=\color{white!15!black}},
ylabel={Fraction of samples},
xlabel={Cosine similarity},
axis background/.style={fill=white},
title = { $ d \pi_V \left( {\rm grad}^{{\mathcal M}} \, 
D(q \| \cdot ) \right)$}
]
\addplot[ybar interval, fill=mycolor1, fill opacity=0.6, draw=black, area legend] table[row sep=crcr] {%
x	y\\
-1	0.00564\\
-0.9	0.0058\\
-0.8	0.00578\\
-0.7	0.0078\\
-0.6	0.0085\\
-0.5	0.0101\\
-0.4	0.01332\\
-0.3	0.0155\\
-0.2	0.01934\\
-0.1	0.0235\\
0	0.03248\\
0.1	0.03714\\
0.2	0.04208\\
0.3	0.04944\\
0.4	0.0573\\
0.5	0.06982\\
0.6	0.09016\\
0.7	0.10888\\
0.8	0.1256\\
0.9	0.27182\\
1	0.27182\\
};
\end{axis}
\end{tikzpicture}%
	\caption{Histograms showing the cosine similarity between $ d \pi_V \left( {\rm grad}^{{\mathcal M}} \, D({\mathcal Q} \| \cdot ) \right)$ and $  {\rm grad}^{{\mathcal M}_V} D(p^\ast \| \cdot)$ (left) and the cosine similarity between $ d \pi_V \left( {\rm grad}^{{\mathcal M}} \, D(q \| \cdot ) \right)$ and ${\rm grad}^{{\mathcal M}_V} D(p^\ast \| \cdot)$ (right), both with respect to the Fisher-Rao metric.}    
	\label{fig:hist_comparison}
\end{figure}

\begin{corollary} \label{mainth}
Let ${\mathcal M}$ be a cylindrical model in ${\mathcal P}$, let $p \in {\mathcal M}$ be admissible, and let $q \in {\mathcal Q}$.
Then 
\begin{equation} 
d\pi_V \left( {\rm grad}^{{\mathcal M}}_p \, {\rm ELBO}(q , \cdot ) \right) 
           \;  = \;   
           {\rm grad}_{\pi_V(p)}^{{\mathcal M}_V} \, 
  {\rm EVIDENCE} . 
           \label{invarianceelbo}
\end{equation}
In particular, the invariance (\ref{invarianceelbo}) holds in all points of the maximal model ${\mathcal M} = {\mathcal P}$ where ${\mathcal M}_V = {\mathcal P}_V$.  
\end{corollary} 
 
The central insight that underlies this result is the following. Under a certain condition, even though the evidence lower bound ``lives'' in an extended space and provides a bound for the evidence, it is equivalent to it in terms of the natural gradient. The gap does not play any role here. 
The condition is met, in particular, if we evaluate the gradients on the corresponding maximal models ${\mathcal P}$ and ${\mathcal P}_V$. If we replace these maximal models by ${\mathcal M}$ and ${\mathcal M}_V$, respectively, then we have to impose a quite strong assumption on ${\mathcal M}$ for the equivalence to hold. Therefore, our result has a conceptual rather than a direct methodological value. It states that in the absence of constraints by a model, the evidence lower bound does not alter the original optimization at all. This is remarkable and demonstrates the consistency of the information-geometric structures, which involve the Fisher-Rao metric and the 
KL-divergence on ${\mathcal P}$ and ${\mathcal P}_V$. 
Any deviation from the invariance is caused by the restriction of the optimization to a model. 
\medskip

We now summarize the path that we pursued in this article by means of the following diagram:
\[
\begin{xy}
\xymatrix{
   D({\mathcal Q} \| \cdot ) \ar@{~>}[r] &  D(q \| \cdot) \ar@{~>}[r]  & {\rm ELBO}(q , \cdot)   & \hspace{-15mm}\mbox{lifted objective functions on ${\mathcal M}$}               \\
   \\
  D(p^\ast \| \cdot)  \ar@{~>}[r] 
 \ar@/^10mm/@{<->}[uu]|{{\rm Theorem} \; \ref{graddeppr}} 
 \ar@{~>}[uu]
  \ar@{<->}[ruu]|{{\rm Theorem} \; \ref{mainthdist}} 
   &  {\rm EVIDENCE} \ar@{<->}[ruu]|{{\rm Corollary} \; \ref{mainth}} & &   
   \hspace{-8mm} \mbox{primary objective functions on ${\mathcal M}_V$}   
   }
   \end{xy}
\]
The overall intention was to relate the maximization of the evidence to the maximization of the evidence lower bound. This has been achieved with Corollary \ref{mainth}. To get there, we have translated the problem to the information-geometric setting, where the primary objective function is given by the KL-divergence $D(p^\ast \| \cdot)$ defined on ${\mathcal M}_V$. This has then been modified in two steps. In the first step, we replaced the primary objective  function by $D({\mathcal Q} \| \cdot)$ defined on ${\mathcal M}$. The interplay between $D(p^\ast \| \cdot)$ and $D({\mathcal Q} \| \cdot)$ was subject of Theorem \ref{graddeppr}. In the second step, we then replaced $D({\mathcal Q} \| \cdot)$
by $D(q \| \cdot)$ which is also defined on ${\mathcal M}$. The interplay of $D(q \| \cdot)$ and the primary objective function was the subject of Theorem \ref{mainthdist}. In a somewhat parallel story line, we can translate this theorem to a statement about the evidence and its lower bound. This is the subject of Corollary \ref{mainth}. Note that the main results of this article are crucially dependent on the information-geometric structures, suggesting that the variational gap does not alter the original objective of learning too much if the corresponding algorithms are based on the natural gradient. The standard Euclidean gradient, on the other hand, depends on the parametrization of the given model and therefore does not yield such parametrization-independent results. An example demonstrating this is presented in the Appendix \ref{sec:euclid}. However, empirical case studies in higher dimensions are required for more conclusive statements.
\medskip

We conclude this section with Remarks \ref{rem1} and \ref{rem2} which outline further research directions as possible continuations of the present work. 

\begin{remark} \label{rem1}
Theorem \ref{mainthdist} states that the gap $D(q \| \pi_{\mathcal Q}(p))$ has no effect on the learning of the primary objective function $D(p^\ast \| \cdot)$, if the model is cylindrical. It is remarkable that this statement is independent of the choice of $q$ so that no adjustment of $q$ is required. However, if the model is not cylindrical, then the gap {\em will\/} have an effect on the learning. In that case, one could try to adjust $q$ in such a way that the effect is minimal. A natural way to do so is by moving $q$ towards the projection $\pi_{\mathcal Q} (p)$. Ideally, the gap will then vanish and the objective function $D(q \| p)$ reduces to $D({\mathcal Q} \| p )$. However, in a typical learning scenario, $q$ is constrained to a so-called {\em recognition model\/} ${\mathcal Q}'$ which is much smaller dimensional than the data manifold ${\mathcal Q}$ defined by (\ref{datamanifold}). Denoting a minimizer of $D(q \| p)$ with respect to $q \in {\mathcal Q}'$ by $\pi_{{\mathcal Q}'}(p)$, we consider the {\em residual gap\/} 
\[
   D(\pi_{{\mathcal Q}'}(p) \| 
   \pi_{{\mathcal Q}}(p) ). 
\]
This gap vanishes for all $p \in {\mathcal M}$, if ${\mathcal Q}'$ is sufficient in the sense that it already contains all projections of points $p \in {\mathcal M}$ onto the maximally possible recognition model, the data manifold ${\mathcal Q}$, that is 
\[
  {\mathcal Q}' \; \supseteq \; 
  \left\{ 
     \pi_{\mathcal Q}(p) \; : \; 
     p \in {\mathcal M}
  \right\}.
\]
In principle, such a recognition model ${\mathcal Q}'$ has the dimensionality of ${\mathcal M}$. However, representing it in terms of a graphical model typically  leads to a blowup of dimensionality \citep{inversion1, inversion2}, 
which forces us to consider smaller recognition models ${\mathcal Q}'$ for learning with a non-vanishing and even large residual gap. The present work suggests, on the other hand, that even in this case, the effect on the learning of the primary objective function $D(p^\ast \| \cdot )$ can be rather small, if the model is close to being cylindrical. That opens up a way to define concise recognition models ${\mathcal Q}'$ with limited perturbation of the primary optimization problem.       
\end{remark}

\begin{remark} \label{rem2}
In this remark, we outline a way to extend the analysis of the present article to the general situation where the model ${\mathcal M}$ is not assumed to be cylindrical. (Note that all our results for non-cylindrical models refer to a particular example and are numerical in nature.) For that, we require the notion of a {\em cylindrical extension\/} $\widetilde{\mathcal M}$ of ${\mathcal M}$. This is a submodel of ${\mathcal P}$ that satisfies the following conditions:
\[
  {\rm (a)} \quad 
  {\mathcal M} \subseteq \widetilde{\mathcal M}, \qquad {\rm (b)} \quad \pi_V ({\mathcal M}) = \pi_V (\widetilde{\mathcal M}), \qquad 
  {\rm (c) \quad \mbox{$\widetilde{\mathcal M}$ is cylindrical}.}   
\]
It is easy to show that any model ${\mathcal M}$ in ${\mathcal P}$ has a cylindrical extension. For instance, we can simply consider the set 
\[
  \widetilde{\mathcal M} \; = \; 
  \left\{ p \in {\mathcal P} \; : \; \pi_V(p) \in {\mathcal M}_V  \right\},
\]
which is the maximal cylindrical extension of ${\mathcal M}$ with respect to set inclusion. We can now apply Theorem \ref{mainthdist} to a cylindrical extension and obtain 
\begin{eqnarray*}
 d\pi_V \left( {\rm grad}^{\widetilde{\mathcal M}}_p \, D(q\|  \cdot) \right) & =&  d\pi_V \left( {\rm grad}^{\mathcal M}_p \, D(q \| \cdot) \right) +  d\pi_V \left( {\rm grad}^{\perp}_p \, D(q \| \cdot) \right) \\
 & = &  {\rm grad}^{{\mathcal M}_V}_{\pi_V(p)} \, D(p^\ast \| \cdot) ,
\end{eqnarray*}
where ${\rm grad}^{\perp}_p \, D(q \| \cdot)$ denotes the projection of ${\rm grad}^{\widetilde{\mathcal M}}_p \, D(q \| \cdot)$ onto the orthogonal complement of $T_p {\mathcal M}$ in $T_p \widetilde{\mathcal M}$. This finally gives us the following generalization of
(\ref{extension}):
\begin{equation}
  d\pi_V \left( {\rm grad}^{\mathcal M}_p \, D(q \| \cdot) \right) \; = \; 
    {\rm grad}^{{\mathcal M}_V}_{\pi_V(p)} \, D(p^\ast \| \cdot) - d\pi_V \left( {\rm grad}^{\perp}_p \, D(q \| \cdot) \right).
   \label{general}
\end{equation}
Equation (\ref{general}) suggests a way to establish a relation between the natural gradient of $D(q \| \cdot)$ and the natural gradient of the primary objective function $D(p^\ast \| \cdot)$ in the general case, with no restriction to cylindrical models. 
\end{remark}

\section{Simplification of the Learning in the Extended Space} \label{bayesnet}

In this article, we have studied the optimization of a primary objective function defined on a model ${\mathcal M}_V$, where $V$ denotes the set of visible units. We have compared this optimization with a corresponding optimization on an extended model ${\mathcal M}$ which incorporates  hidden units. More precisely, the former objective function on ${\mathcal M}_V$ is the mean evidence, whereas the latter is given by the evidence lower bound defined on ${\mathcal M}$. We have stated that the replacement of the primary objective function by a lower bound can greatly simplify the optimization process. In this section, we are now going to provide an instance of this simplification in the context of Bayesian graphical models. Such a model is defined in terms of a directed acyclic graph $G = (N,E)$ with node set $N = V \cup H$. The points of the corresponding Bayesian graphical model, which we denote by ${\mathcal P}^G$, are those probability distributions $p$ in ${\mathcal P}$ that factorize according to $G$, that is 
\begin{equation*}
     p(x) \,=\, \prod_{s\in N} p(x_s|x_{\pa(s)}).
\end{equation*}
Here, $\pa(s)$ denotes the parents of unit $s$, those units $r \in N$ for which $(r,s) \in E$. Typically, each conditional probability distribution $p(x_s|x_{\pa(s)})$, which we interpret as the local generative mechanism of unit $s$, is parametrized in terms of a local parameter vector $\theta_s =(\theta_{s,1}, \dots, \theta_{s,d_s})$, indicated by $p(x_s|x_{\pa(s)}; \theta_s)$.  
Concatenating all the parameter vectors $\theta_s$ to one vector $\theta = (\theta_s)_{s \in N}$ of size $d = \sum_{s \in N} d_s$, the overall probability distribution is 
parametrized as  
\begin{equation} \label{prodst}
     p_\theta (x) \, := \, p(x ; \theta) 
     \, := \, \prod_{s\in N} p(x_s|x_{\pa(s)}; \theta_s).
\end{equation} 
The model ${\mathcal M}$, given by the collection $p(\cdot ; \theta)$, $\theta \in \Theta \subseteq {\Bbb R}^d$, is a submodel of ${\mathcal P}^G$. Whenever referring to a submodel of a Bayesian graphical model in the following, we mean this kind of a submodel without explicitly mentioning it.  
The  product structure (\ref{prodst}) implies a number of simplifications which have been discussed in \citep{ay2020locality}. In this article, we add a somewhat simple but illuminating point to this discussion.    
\medskip

In order to compute the gradient of an objective function in a non-singular point $p$ of ${\mathcal M}$, we need to project the corresponding gradient in ${\mathcal P}$ onto the tangent space $T_p {\mathcal M}$ of ${\mathcal M}$ in $p$, in terms of the orthogonal projection $\Pi_p$. We have encountered this method of determining the gradient several times in this article.
The projection $\Pi_p$ will be particularly simple, if the tangent space decomposes into orthogonal lower-dimensional spaces. We are now going to highlight this structure for any submodel ${\mathcal M}$ of a Bayesian graphical model. For $\theta \in \Theta$, the tangent space of ${\mathcal M}$ in $p_\theta$, denoted by $T_{p_\theta} {\mathcal M}$, is typically expressed in terms of the derivatives 
\begin{equation*}
   \partial_{s,k} (\theta) 
   \; = \; \sum_x \partial_{s,k} (x; \theta) \, \delta^x ,   
\end{equation*}
with
\begin{eqnarray}     
\partial_{s,k} (x; \theta)  
& := &  \der{\theta_{s,k}} \, p(x ; \theta) \nonumber \\
& = & p(x ; \theta) \, \der{\theta_{s,k}} \, \ln p(x ; \theta)
    \nonumber \\
& = & p(x ; \theta) \, \der{\theta_{s,k}} \, \ln p(x_s | x_{{\rm pa}(s)} ; \theta_s). \label{tangentvec}    
\end{eqnarray}   
The vectors $\partial_{s,k} (\theta)$, $s \in N$, $k = 1, \dots , d_s$, are clearly contained in $T_{p_\theta} {\mathcal M}$ but in general they need not to span $T_{p_\theta} {\mathcal M}$. To express all elements of the tangent space, we assume that the parametrization $\theta \mapsto p_\theta = p(\cdot ; \theta)$ is
{\em proper\/} in the sense that the vectors $\partial_{s,k} (\theta)$, $s \in N$, $k = 1,\dots, d_s$, span $T_{\theta} {\mathcal M}$. On the other hand, assuming a proper parametrization, we cannot expect these vectors to form a basis of $T_{p_\theta} {\mathcal M}$ as they do not have to be linearly inedependent. For a submodel ${\mathcal M}$ of a Bayesian graphical model, however, these vectors give rise to a natural orthogonal decomposition, which simplifies the projection onto the tangent space $T_{p_\theta} {\mathcal M}$. 

\begin{proposition} \label{orthogonalVectors} Let $\cM$ be a submodel of a Bayesian graphical model, parametrized by $\theta$ (not necessarily by a proper parametrization). Then, for $s \neq t$, $1\leq k \leq d_s$, $1\leq l \leq d_t$, we have 
     \begin{equation*}
    g_\theta^{\rm FR} \left( \partial_{s,k}(\theta), \partial_{t,l}(\theta) \right) 
     \,=\, 0.
     \end{equation*}
Assuming that the parametrization $\theta \mapsto p_\theta$ is proper, we obtain an orthogonal decomposition of the tangent space $T_{p_\theta} {\mathcal M}$ into the subspaces   
\begin{equation*}
 T^{(s)}_\theta {\mathcal M} \; := \; 
 {\rm span}\{ \partial_{s,k} (\theta) \; : \; k = 1,\dots, d_s\}, \quad s \in N.
\end{equation*}
\end{proposition}
See Appendix~\ref{sec:proof_proposition_ortho_vectors} for a proof. 

% Acknowledgements and Disclosure of Funding should go at the end, before appendices and references

%\newpage
\acks{NA and JvO  acknowledge the support of the Deutsche Forschungsgemeinschaft Priority
Programme ``The Active Self'' (SPP 2134).}

% Manual newpage inserted to improve layout of sample file - not
% needed in general before appendices/bibliography.

% \newpage 
\appendix
\section{Examples of Cylindrical and Non-Cylindrical Models} \label{sec:examples}
\subsection*{Example 1: The Independence Model}

Let us consider the setting in which $\cP$ is the set of distributions over two binary nodes $s$ and $t$. The state space is given by
\begin{eqnarray*}
    \sX &=& \sX_s \times \sX_t,\\
    \sX_s &=& \sX_t \,=\, \{0,1\},
\end{eqnarray*}
and we let $X_r:\sX \to \sX_r, r \in \{s, t\}$ be the projections.
The marginalization map is given by
\begin{eqnarray}\label{eq:dpi_eg1}
    \pi_V\colon \cP &\to& \cP_V, \nonumber\\
    p(x_s,x_t) &\mapsto& p(x_t) \; = \; \sum_{x_s} p(x_s,x_t) ,
\end{eqnarray}
and its differential
\begin{eqnarray*}
    d\pi_V\colon T_p\cP &\to& T_{\pi_V(p)}\cP_V, \\
    A \; = \; \sum_{x_s,x_t} A(x_s,x_t) \delta^{(x_s,x_t)} &\mapsto& \sum_{x_t} \left(\sum_{x_s} A(x_s,x_t)\right) \delta^{x_t}.
\end{eqnarray*}
The vertical and horizontal spaces are given by  
\begin{align}
    &\begin{aligned}
        \cV_p  &\; = \; \mathrm{ker}\,d \pi_V \\
        &\; = \; \left\{ A \in T_p\cP : \sum_{x_s} A(x_s, x_t) \; = \; 0, x_t \in \sX_t \right\} \\
        &\; = \; \vspan \left\{ \delta^{(0,0)} - \delta^{(1,0)}, \delta^{(0,1)} - \delta^{(1,1)} \right\}, 
    \end{aligned}\nonumber  \\    
     &\begin{aligned}
     \cH_p &\; = \;
     (\mathrm{ker} \, d\pi_V )^\perp \\
     &\; = \; \left\{ A \in T_p\cP : \frac{A(0, x_t)}{p(0,x_t)} - \frac{A(1, x_t)}{p(1,x_t)}  \; = \; 0, x_t \in \sX_t \right\} \\
     &\; = \; \vspan \left\{ p(0,0) \delta^{(0,0)} + p(1,0) \delta^{(1,0)}, p(0,1) \delta^{(0,1)}  + p(1,1) \delta^{(1,1)} \right\}. 
    \end{aligned} 
    \label{eq:horizontal2dim}
\end{align} 
Now we let the model be the independence model, given by 
\begin{align*}
    \cM \; = \; \{ p \in \cP: p(x_s, x_t) \; = \; p(x_s)p(x_t) \}.
\end{align*}
\begin{figure}[h!]
\centering
\begin{tikzpicture}[neuron/.style={circle,draw, minimum size=.7cm, inner sep=0}]
    \node (a) [neuron] at (0,0) {$s$};
    \node (b) [neuron] at (0,-1.2) {$t$};
\end{tikzpicture}
\caption{Graph $G$}
\label{fig:bayesian-net1}
\end{figure}
\\
\noindent This model factorizes over the graph depicted in Figure \ref{fig:bayesian-net1}
and can be parameterized as follows:
\begin{align*}
    p(X_s=1; \theta) &\; = \; \theta_s, \\ 
    p(X_t=1; \theta) &\; = \; \theta_t.
\end{align*}
This parametrization gives
\begin{align*}
    p_\theta \; = \; (1-\theta_s) (1-\theta_t) \delta^{(0,0)} + (1-\theta_s) \theta_t \delta^{(0,1)} + \theta_s (1-\theta_t) \delta^{(1,0)} + \theta_s \theta_t \delta^{(1,1)}.
 \end{align*}
 The tangent space  $T_{p_\theta}\cM$ is spanned by the parameter tangent vectors given by 
 \begin{align*}
     \partial_s(\theta) &\; = \; - (1-\theta_t) \delta^{(0,0)} - \theta_t \delta^{(0,1)} +  (1-\theta_t) \delta^{(1,0)} + \theta_t \delta^{(1,1)}, \\
     \partial_t(\theta) &\; = \; - (1-\theta_s) \delta^{(0,0)} + (1 - \theta_s) \delta^{(0,1)}  -\theta_s \delta^{(1,0)} + \theta_s \delta^{(1,1)}.
 \end{align*}
Note that 
\begin{align*}
    \partial_s(\theta) &\; = \; -(1-\theta_t) \left(\delta^{(0,0)} - \delta^{(1,0)}\right) - \theta_t \left( \delta^{(0,1)} - \delta^{(1,1)}\right) \in \cV_{p_\theta}, \\
    \partial_t(\theta) &\; = \; -\frac{1}{(1-\theta_t)} \left(p_\theta(0,0) \delta^{(0,0)} + p_\theta(1,0) \delta^{(1,0)}\right) +\frac{1}{\theta_t} \left(p_\theta(0,1) \delta^{(0,1)}  + p_\theta(1,1) \delta^{(1,1)}\right) \in \cH_{p_\theta}.
\end{align*}
Therefore, this model is cylindrical.

\subsection*{Example 2: Non-Cylindrical Two-Node Model}
\begin{figure}[ht]
    \centering
    \begin{tikzpicture}[neuron/.style={circle,draw, minimum size=.7cm, inner sep=0}]
        \node (a) [neuron] at (0,0) {$s$};
        \node (b) [neuron] at (0,-1.2) {$t$};
        \draw [->] (a) to (b);
    \end{tikzpicture}
    \caption{Graph $G$}
    \label{fig:bayesian-net2}
\end{figure}
For the same $\cP$ and the same $\pi_V$ as in Example 1 (see equation \eqref{eq:dpi_eg1}), let us fix a distribution $\bar{p}(x_s)$ and consider the following model:
\begin{align*}
    \cM \; = \; \{ p \in \cP : p(x_s, x_t) = \bar{p}(x_s)p(x_t|x_s) \}, 
\end{align*}
which factorizes over the graph from Figure \ref{fig:bayesian-net2}. This model can be parametrized by 
\begin{align*}
    &p(X_t = 1 | X_s = 0; \theta) \; = \; \theta_{t,1},\\
    &p(X_t = 1 | X_s = 1; \theta) \; = \; \theta_{t,2}.
\end{align*}
This parametrization gives
\begin{align*}
    p_\theta \; = \; (1-\theta_{t,1}) \bar{p}(0) \delta^{(0,0)} + \theta_{t,1} \bar{p}(0) \delta^{(0,1)} +  (1-\theta_{t,2}) \bar{p}(1) \delta^{(1,0)} + \theta_{t,2} \bar{p}(1) \delta^{(1,1)}.
\end{align*}
The parameter tangent vectors of $T_{p_\theta}\cM$ are given by 
\begin{align*}
    &\partial_1(\theta) \; = \; - \bar{p}(0) \delta^{(0,0)} + \bar{p}(0) \delta^{(0,1)}, \\
    &\partial_2(\theta) \; = \; - \bar{p}(1) \delta^{(1,0)} + \bar{p}(1) \delta^{(1,1)}.
\end{align*}
For the intersection of $T_{p_\theta}\cM$ with $\cV_{p_\theta}$ we have
\begin{align*}
    T_{p_\theta}\cM \cap \cV_{p_\theta} \; = \; \vspan \left\{ \frac{1}{\bar{p}(0)} \partial_1(\theta) - \frac{1}{\bar{p}(1)} \partial_2(\theta) \right\}.
\end{align*}
Note that this space is only one-dimensional. In order for $\cM$ to be cylindrical, we would therefore need that the intersection of $T_{p_\theta}\cM$ with $\cH_{p_\theta}$ is non-trivial. Let us assume by contradiction that there exists $\alpha, \beta$ such that $\alpha \partial_1(\theta) + \beta \partial_2(\theta) \in \cH_{p_\theta}$. WLOG assume $\alpha = 1$. Using the definition of $\cH_{p_\theta}$ from equation \eqref{eq:horizontal2dim}, we get the following conditions:
\begin{align*}
    \frac{-\bar{p}(0)}{p_\theta(0,0)} \; = \; \beta \frac{-\bar{p}(1)}{p_\theta(1,0)}
\end{align*}
and,
\begin{align*}
    \frac{\bar{p}(0)}{p_\theta(0,1)} \; = \; \beta \frac{\bar{p}(1)}{p_\theta(1,1)}.
\end{align*}
Working out these conditions gives $\beta = \frac{1-\theta_{t,2}}{1-\theta_{t,1}}$ and $\beta = \frac{\theta_{t,2}}{\theta_{t,1}}$ respectively. Therefore, we conclude that this model is only cylindrical in the points where $\theta_{t,1} = \theta_{t,2}$ which are exactly the points for which $s$ and $t$ are independent, and is in general not cylindrical.

\subsection*{Example 3: Non-Cylindrical Three-Node Model}

Now, let $\cP$ be the space of probability measures over the sample space $\sX$ given by
\begin{align*}
    \sX &\; = \; \sX_s \times \sX_{t_1} \times \sX_{t_2},\\
    \sX_s &\; = \; \sX_{t_1} \; = \; \sX_{t_2} \; = \; \{0,1\}.
\end{align*}
We let $X_r:\sX \to \sX_r, r \in \{s, t_1, t_2\}$  be the projections.
The marginalization map is given by
\begin{eqnarray*}
    \pi_V\colon \cP &\to& \cP_V,\\
    p(x_s,x_{t_1},x_{t_2}) &\mapsto& p(x_{t_1},x_{t_2}) = \sum_{x_s} p(x_s,x_{t_1},x_{t_2}),
\end{eqnarray*}
and its differential
\begin{eqnarray}\label{eq:dpi_eg3}
    d\pi_V\colon T_p\cP &\to& T_{\pi_V(p)}\cP_V, \nonumber \\
    A = \sum_{x_s,x_{t_1},x_{t_2}} A(x_s,x_{t_1},x_{t_2}) \delta^{(x_s,x_{t_1},x_{t_2})} &\mapsto& \sum_{x_{t_1},x_{t_2}} \left(\sum_{x_s} A(x_s,x_{t_1},x_{t_2})\right) \delta^{(x_{t_1},x_{t_2})}.
\end{eqnarray}
The vertical and horizontal spaces are given by  
\begin{align*}
    \cV_p & \; = \; \mathrm{ker} \, d\pi_V = \left\{ A \in T_p\cP : \sum_{x_s} A(x_s,x_{t_1},x_{t_2}) = 0 \right\}, \\
    \cH_p &\; = \; (\mathrm{ker}\, d\pi_V)^\perp = \left\{ A \in T_p\cP : \sum_{x_s} \frac{A(x_s,x_{t_1},x_{t_2})}{p(x_s,x_{t_1},x_{t_2})} (-1)^{x_s}  = 0 \right\}.
\end{align*}
Now we consider the model given by
\begin{align*}
    \cM \; = \; \{ p \in \cP : p(x_s, x_{t_1}, x_{t_2}) = p(x_s)p(x_{t_1}|x_s)p(x_{t_2}|x_s) \}. 
\end{align*}
Note that this model is both equal to the Bayesian graphical model of distributions that factorise over the graph $G$, and equal to the distributions corresponding to the Boltzmann machine with the undirected graph $G^\sim$, both in Figure \ref{fig:bayesian-net3}. 
\begin{figure}[]
\centering
\begin{tikzpicture}
    \node at (0,0) {
    \begin{tikzpicture}[neuron/.style={circle,draw, minimum size=.7cm, inner sep=0}]
        \node (g) at (-1.3,.8) {$G$};
        \node (a) [neuron] at (0,0) {$s$};
        \node (b) [neuron] at (-1,-.7) {$t_1$};
        \node (c) [neuron] at (1,-.7) {$t_2$};
        \draw [->] (a) -- (b);
        \draw [->] (a) -- (c);
    \end{tikzpicture}
    };
    \node at (5,0) {
    \begin{tikzpicture}[neuron/.style={circle,draw, minimum size=.7cm, inner sep=0}]
        \node (g) at (-1.3,.8) {$G^\sim$};
        \node (a) [neuron] at (0,0) {$s$};
        \node (b) [neuron] at (-1,-.7) {$t_1$};
        \node (c) [neuron] at (1,-.7) {$t_2$};
        \draw [] (a) -- (b);
        \draw [] (a) -- (c);
    \end{tikzpicture}};
\end{tikzpicture}
\caption{(left) Directed graph $G$; (right) Undirected graph $G^\sim$.}
\label{fig:bayesian-net3}
\end{figure}
The model can be parameterized as follows:
\begin{align*}
    p(X_s=1; \theta) &\; = \; \theta_s, \\ 
    p(X_{t_1}=1|X_s=0; \theta) &\; = \; \theta_{t_1,1},\\
    p(X_{t_1}=1|X_s=1; \theta) &\; = \; \theta_{t_1,2},\\
    p(X_{t_2}=1|X_s=0; \theta) &\; = \; \theta_{t_2,1},\\
    p(X_{t_2}=1|X_s=1; \theta) &\; = \; \theta_{t_2,2}.
\end{align*}
As in the previous examples, this parametrization gives
\begin{align*}
    p_\theta \; = \; (1-\theta_s)& (1 - \theta_{t_1,1}) (1 - \theta_{t_2,1}) \delta^{(0,0,0)} + 
        (1-\theta_s) (1 - \theta_{t_1,1}) \theta_{t_2,1} \delta^{(0,0,1)}\\
        &+(1-\theta_s) \theta_{t_1,1}       (1 - \theta_{t_2,1}) \delta^{(0,1,0)} + 
        (1-\theta_s) \theta_{t_1,1} \theta_{t_2,1} \delta^{(0,1,1)} \\
        &+\theta_s (1 - \theta_{t_1,2}) (1 - \theta_{t_2,2}) \delta^{(1,0,0)} +
        \theta_s (1 - \theta_{t_1,2}) \theta_{t_2,2} \delta^{(1,0,1)} \\
        &+\theta_s \theta_{t_1,2}       (1 - \theta_{t_2,2}) \delta^{(1,1,0)} +
        \theta_s \theta_{t_1,2} \theta_{t_2,2} \delta^{(1,1,1)}.
 \end{align*}
To simplify the ensuing long expressions, we next identify the space of signed measures on $\sX$ with $\bR^8$, where we use the following enumeration of the sample space $\sX$: 
\begin{align*}
    ((0,0,0), (0,0,1), (0,1,0), (0,1,1), (1,0,0), (1,0,1), (1,1,0), (1,1,1) ). 
\end{align*}
This gives for example
\begin{align*}
    \delta^{(0,1,0)} \; = \; \begin{bmatrix}
        0 \\ 0 \\ 1 \\ 0 \\ 0 \\ 0 \\ 0 \\ 0 
    \end{bmatrix}
    \quad
    \textnormal{ and }
    \quad
    p_\theta \; = \; \begin{bmatrix}
        (1-\theta_s) (1 - \theta_{t_1,1}) (1 - \theta_{t_2,1}) \\
        (1-\theta_s) (1 - \theta_{t_1,1}) \theta_{t_2,1} \\
        (1-\theta_s) \theta_{t_1,1}       (1 - \theta_{t_2,1}) \\
        (1-\theta_s) \theta_{t_1,1} \theta_{t_2,1} \\
        \theta_s (1 - \theta_{t_1,2}) (1 - \theta_{t_2,2}) \\
        \theta_s (1 - \theta_{t_1,2}) \theta_{t_2,2} \\
        \theta_s \theta_{t_1,2}       (1 - \theta_{t_2,2}) \\
        \theta_s \theta_{t_1,2} \theta_{t_2,2} 
    \end{bmatrix}.
\end{align*}
The parameter tangent vectors of $T_p\cM$ can similarly be identified as
\begin{align*}
    \partial_s(\theta) \; = \; &\begin{bmatrix}
        - (1 - \theta_{t_1,1}) (1 - \theta_{t_2,1}) \\
        - (1 - \theta_{t_1,1}) \theta_{t_2,1} \\
        - \theta_{t_1,1}       (1 - \theta_{t_2,1}) \\
        - \theta_{t_1,1} \theta_{t_2,1} \\
       (1 - \theta_{t_1,2}) (1 - \theta_{t_2,2}) \\
       (1 - \theta_{t_1,2}) \theta_{t_2,2} \\
       \theta_{t_1,2}       (1 - \theta_{t_2,2}) \\
       \theta_{t_1,2} \theta_{t_2,2} 
    \end{bmatrix}, \\
    \partial_{t_1,1}(\theta) \; = \; &\begin{bmatrix}
        - (1-\theta_s)  (1 - \theta_{t_2,1}) \\
        - (1-\theta_s)  \theta_{t_2,1} \\
        (1-\theta_s) (1 - \theta_{t_2,1}) \\
        (1-\theta_s) \theta_{t_2,1} \\
       0 \\ 0 \\ 0 \\ 0
    \end{bmatrix}, \quad 
    \partial_{t_1,2}(\theta) \; = \; \begin{bmatrix}
        0 \\ 0 \\ 0 \\ 0 \\
        - \theta_s  (1 - \theta_{t_2,2}) \\
        - \theta_s  \theta_{t_2,2} \\
        \theta_s (1 - \theta_{t_2,2}) \\
        \theta_s  \theta_{t_2,2} 
    \end{bmatrix},
\end{align*}
\begin{align*}
    \partial_{t_2,1}(\theta) \; = \; &\begin{bmatrix}
        - (1-\theta_s) (1 - \theta_{t_1,1})  \\
        (1-\theta_s) (1 - \theta_{t_1,1})  \\
        - (1-\theta_s) \theta_{t_1,1}        \\
        (1-\theta_s) \theta_{t_1,1}  \\
        0 \\ 0 \\ 0 \\ 0
    \end{bmatrix}, \quad 
    \partial_{t_2,2}(\theta) \; = \; \begin{bmatrix}
        0 \\ 0 \\ 0 \\ 0 \\
        - \theta_s (1 - \theta_{t_1,2})  \\
        \theta_s (1 - \theta_{t_1,2})  \\
        - \theta_s \theta_{t_1,2}        \\
        \theta_s \theta_{t_1,2}  
    \end{bmatrix}.
\end{align*}
We let $B_{p_\theta}\in \bR^{8\times 5}$ be the matrix with these parameter vectors as columns, i.e., 
\begin{align*}
    B_{p_\theta} \; = \; \begin{bmatrix}
        | & | & | & | & | \\
     \partial_s(\theta) & \partial_{t_1, 1}(\theta) & \partial_{t_1, 2}(\theta) & \partial_{t_2, 1}(\theta) & \partial_{t_2, 2}(\theta) \\
     | & | & | & | & |
    \end{bmatrix}.
\end{align*}
In the same spirit as above, the space of signed measures on $\sX_{t_1} \times \sX_{t_2}$ can be identified with $\bR^4$, where we use enumerate the sample space $\sX_{t_1} \times \sX_{t_2}$ as $((0,0), (0,1), (1,0), (1,1))$. For example, this gives 
\begin{align*}
    \delta^{(1,0)} \; = \; \begin{bmatrix}
        0 \\ 0 \\ 1 \\ 0 
    \end{bmatrix}. 
\end{align*}
With the identification of $p \in \cP$ with vectors in $\bR^8$ and the identification of $p_V \in \cP_V$ with vectors in $\bR^4$, we can identify the map $d\pi_p$ defined in \eqref{eq:dpi_eg3} with a matrix $J\in \bR^{4\times 8}$ given by
\begin{align*}
    J \; = \; \begin{bmatrix}
        1 & 0 & 0 & 0 & 1 & 0 & 0 & 0 \\
        0 & 1 & 0 & 0 & 0 & 1 & 0 & 0 \\
        0 & 0 & 1 & 0 & 0 & 0 & 1 & 0 \\
        0 & 0 & 0 & 1 & 0 & 0 & 0 & 1
    \end{bmatrix}.
\end{align*}
Note that 
\begin{align*}
    \dim \left(T_{p_\theta}\cM \cap \cV_{p_\theta} \right) \; = \; \dim \left( \ker \, J B_{p_\theta} \right). 
\end{align*}
Similarly, one can derive 
\begin{align*}
    \dim \left(T_{p_\theta}\cM \cap \cH_{p_\theta} \right) \; = \; \dim \left( \ker \, \tilde{J} G_{p_\theta} B_{p_\theta} \right), 
\end{align*}
where
\begin{align*}
    \tilde{J} \; = \; \begin{bmatrix}
        1 & 0 & 0 & 0 & -1 & 0 & 0 & 0 \\
        0 & 1 & 0 & 0 & 0 & -1 & 0 & 0 \\
        0 & 0 & 1 & 0 & 0 & 0 & -1 & 0 \\
        0 & 0 & 0 & 1 & 0 & 0 & 0 & -1
    \end{bmatrix}
\end{align*}
and $G_{p_\theta}$ is the matrix representative of the Fisher-Rao metric at $p_\theta$, given by 
\begin{align*}
    G_{p_\theta} \; = \; \begin{bmatrix}
        1/p_1 & 0 & 0 & 0 & 0 & 0 & 0 & 0 \\
        0 & 1/p_2 & 0 & 0 & 0 & 0 & 0 & 0 \\
        0 & 0 & 1/p_3 & 0 & 0 & 0 & 0 & 0 \\
        0 & 0 & 0 & 1/p_4 & 0 & 0 & 0 & 0 \\
        0 & 0 & 0 & 0 & 1/p_5 & 0 & 0 & 0 \\
        0 & 0 & 0 & 0 & 0 & 1/p_6 & 0 & 0 \\
        0 & 0 & 0 & 0 & 0 & 0 & 1/p_7 & 0 \\
        0 & 0 & 0 & 0 & 0 & 0 & 0 & 1/p_8
    \end{bmatrix}, 
\end{align*}
with $p_i = p_\theta(x_i)$, where $x_i$ is the $i$th element of the sample space $\sX$. \\

In order to show that this model is not cylindrical, we only have to show this for one specific point. We choose the point $\theta_s = \theta_{t_1, 1} = \theta_{t_2, 1} = 1/2, \theta_{t_1, 2} = \theta_{t_2, 2} = 1/3$. \\
    
For this choice of $\theta$, $B_{p_\theta}$ becomes
\begin{align*}
    B_{p_\theta} \; = \; \begin{bmatrix}
        -1/4  &  -1/4  & 0    & -1/4 & 0    \\
        -1/4  &  -1/4  & 0    & 1/4  & 0    \\
        -1/4  &  1/4   & 0    & -1/4 & 0    \\
        -1/4  &  1/4   & 0    & 1/4  & 0    \\
        4/9   &  0     & -1/3 & 0    & -1/3 \\ 
        2/9   &  0     & -1/6 & 0    & 1/3  \\
        2/9   &  0     & 1/3  & 0    & -1/6 \\
        1/9   &  0     & 1/6  & 0    & 1/6  
    \end{bmatrix}.
\end{align*}
It can be verified that the space $\ker \, JB_{p_\theta}$ is spanned by the following vectors:
\begin{align*}
    \begin{bmatrix}
        3 \\
        0 \\
        1 \\
        1 \\
        0
    \end{bmatrix},
    \begin{bmatrix}
        3 \\
        1 \\
        0 \\
        0 \\
        1
    \end{bmatrix},
\end{align*}
and is therefore two-dimensional. \\

Similarly, the space $\ker \, \tilde{J} G_{p_\theta} B_{p_\theta} $ is spanned by 
\begin{align*}
    \begin{bmatrix}
        -3/16 \\
        9/8 \\
        1 \\
        0 \\
        0
    \end{bmatrix}, 
    \begin{bmatrix}
        -3/16 \\
        0 \\
        0 \\
        9/8 \\
        1
    \end{bmatrix},
\end{align*}
and is therefore also two-dimensional. This means that $(T_{p_\theta}\cM \cap \cV_{p_\theta}) \oplus (T_{p_\theta}\cM \cap \cH_{p_\theta})$ is four-dimensional and therefore unequal to $T_{p_\theta}\cM$ which is five-dimensional. We therefore conclude that $\cM$ is not cylindrical. 

\section{The Natural Versus the Euclidean Gradient of the Variational Gap} \label{sec:euclid} 
Let us again consider the setting in which ${\mathcal P}$ is the set of distributions over two binary nodes $s$ and $t$, where $t$ is the visible and $s$ is the hidden node. The state space is given by 
\begin{align*}
    \sX &\; = \; \sX_s \times \sX_t,\\
    \sX_s &\; = \; \sX_t \; = \; \{0,1\}.
\end{align*}
Thus, we have the four states $(0,0)$, $(1,0)$, $(0,1)$, and $(1,1)$ and the corresponding Dirac measures $\delta^{(0,0)}$, $\delta^{(1,0)}$, $\delta^{(0,1)}$, and $\delta^{(1,1)}$. In what follows, we parametrize ${\mathcal P}$ in terms of 
\[
  \varphi: \; {\Bbb R}^3 \ni \theta = (\theta_1, \theta_2, \theta_3) \;
   \mapsto \; \theta_1 \, \delta^{(0,0)} + \theta_2 \, \delta^{(1,0)} +\theta_3 \, \delta^{(0,1)} + (1 - \theta_1 - \theta_2 - \theta_3 )\, \delta^{(1,1)} \in {\mathcal P},
\]
where we assume $\theta_1,\theta_2,\theta_3 > 0$ and $\theta_1 + \theta_2 + \theta_3 < 1$. To simplify the notation, we abbreviate $1 - \left(\theta_1 + \theta_2 + \theta_3\right)$ by $\theta_4$. The tangent space of ${\mathcal P}$ in $\varphi(\theta)$ is spanned by the basis
\begin{align*}
   \partial_1 (\theta) &\; = \; 
   \frac{\partial \varphi}{\partial \theta_1}(\theta) \; = \;  \delta^{(0,0)} - \delta^{(1,1)}, \\  
   \partial_2 (\theta) &\; = \; 
   \frac{\partial \varphi}{\partial \theta_2}(\theta) \; = \;  \delta^{(1,0)} - \delta^{(1,1)}, \\
   \partial_3 (\theta) &\; = \; 
   \frac{\partial \varphi}{\partial \theta_3}(\theta) \; = \;  \delta^{(0,1)} - \delta^{(1,1)}.
\end{align*}
Application of the differential (\ref{projection}) of the marginalization map $\pi_V: {\mathcal P} \mapsto {\mathcal P}_{V} = {\mathcal P}_{\{t\}}$ gives us 
\begin{align*} 
   d\pi_V(\partial_1(\theta)) & \; = \; \delta^0 - \delta^1, \\
    d\pi_V(\partial_2(\theta)) & \; = \; 0, \\
     d\pi_V(\partial_3(\theta)) & \; = \; \delta^0 - \delta^1. 
\end{align*}
Here, $\delta^0$ and $\delta^1$ denote the Dirac measures of the states $0$ and $1$ of the visible node $t$. 
The Fisher information matrix $G(\theta)$ with components $g^{\rm FR}(\partial_i(\theta), \partial_j(\theta))$ is given as
\begin{align*}
G(\theta) &\; = \; \frac{1}{\theta_4}
\begin{bmatrix}
  \frac{\theta_4}{\theta_1} + 1 & 1 & 1 \\
    1 & \frac{\theta_4}{\theta_2} + 1& 1  \\
    1 & 1 & \frac{\theta_4}{\theta_3} + 1 
\end{bmatrix},
\end{align*}
with inverse 
\begin{align} \label{fisherinverse}
G^{-1}(\theta) &\; = \; 
\begin{bmatrix}
   \theta_1 (1 - \theta_1) & - \theta_1 \theta_2 & - \theta_1 \theta_3 \\
    - \theta_2 \theta_1 & \theta_2 (1 - \theta_2) & - \theta_2 \theta_3  \\
    - \theta_3 \theta_1 & - \theta_3 \theta_2 & \theta_3 (1 - \theta_3) 
\end{bmatrix}.
\end{align}
Given a differentiable function ${\mathcal L}: {\mathcal P} \to {\Bbb R}$, we set
\begin{align*} 
\nabla {\mathcal L}(\theta)
 & :\; = \; 
 \begin{bmatrix} \displaystyle
 \frac{\partial {\mathcal L} \circ \varphi }{\partial \theta_1} (\theta) \\  
 \displaystyle
 \frac{\partial {\mathcal L} \circ \varphi }{\partial \theta_2} (\theta) \\  
 \displaystyle
 \frac{\partial {\mathcal L} \circ \varphi }{\partial \theta_3} (\theta) 
 \end{bmatrix}
\end{align*}
and 
\begin{align*} 
\widetilde{\nabla}{\mathcal L}(\theta)
 & :\; = \; G^{-1}(\theta)\, \nabla 
 {\mathcal L}(\theta) .
\end{align*}
The Euclidean gradient with respect to the standard inner product in ${\Bbb R}^3$ is given by 
\[
   \sum_{i = 1}^3 \left[ \nabla{\mathcal L}(\theta)\right]_i \, \partial_i(\theta), 
\]
whereas the natural gradient involves the Fisher information matrix:
\[
   \sum_{i = 1}^3 \left[ \widetilde{\nabla}{\mathcal L}(\theta)\right]_i \, \partial_i(\theta). 
\]
Learning based on the Euclidean gradient ascent method follows the update rule 
\begin{equation} \label{euclideaniteration}
   \theta_{m + 1} \; = \; \theta_m + \varepsilon \cdot \nabla {\mathcal L}(\theta_m), \qquad m = 0,1,2, \dots,
\end{equation}
whereas the natural gradient method suggests
\begin{equation} \label{naturaliteration}
   \theta_{m + 1} \; = \; \theta_m + \varepsilon \cdot \widetilde{\nabla} {\mathcal L}(\theta_m), \qquad m = 0,1,2, \dots.
\end{equation}
One could apply these iteration rules, for instance, to maximize the evidence and its lower bound, respectively. This article suggests that the replacement of the evidence by its lower bound will have a less ``visible'' effect if we use the natural gradient iteration rule (\ref{naturaliteration}) in comparison with the Euclidean iteration rule (\ref{euclideaniteration}). This can be formally studied by mapping the gradient of the variational gap via the differential $d\pi_V$.        
In what follows, we evaluate the Euclidean as well as the natural gradient of ${\rm GAP}_q := {\rm GAP}(q, \cdot)$, defined by (\ref{gap}). After some straightforward calculations, we obtain
\begin{eqnarray*}
   \left[\nabla \, {\rm GAP}_q(\theta)\right]_1 & = &
 \frac{p^\ast(0)}{\theta_1 + \theta_3} - \frac{p^\ast(1)}{\theta_2 + \theta_4} -  \frac{p^\ast(0) q(0|0)}{\theta_1} + \frac{p^\ast(1) q(1|1)}{\theta_4}, \\
  \left[\nabla \, {\rm GAP}_q(\theta)\right]_2 & = & - \frac{p^\ast(1) q(0|1)}{\theta_2} + \frac{p^\ast(1) q(1|1)}{\theta_4} ,   \\
   \left[\nabla \, {\rm GAP}_q(\theta)\right]_3 & = & \frac{p^\ast(0)}{\theta_1 + \theta_3} - \frac{p^\ast(1)}{\theta_2 + \theta_4} -  \frac{p^\ast(0) q(1|0)}{\theta_3} + \frac{p^\ast(1) q(1|1)}{\theta_4}.   
\end{eqnarray*}
With the inverse of the Fisher information matrix, (\ref{fisherinverse}), this yields  
\begin{eqnarray*}
   \left[\widetilde{\nabla} \, {\rm GAP}_q(\theta)\right]_1 & = &
 \left(\frac{p^\ast(0)}{\theta_1 + \theta_3} - \frac{p^\ast(1)}{\theta_2 + \theta_4}\right) \theta_1 (\theta_2 + \theta_4) - 
 p^\ast(0)q(0|0) + \theta_1, \\
  \left[\widetilde{\nabla} \, {\rm GAP}_q(\theta)\right]_2 & = & - 
  \left(\frac{p^\ast(0)}{\theta_1 + \theta_3} - \frac{p^\ast(1)}{\theta_2 + \theta_4}\right) \theta_2 (\theta_1 + \theta_3) - 
 p^\ast(1)q(0|1) + \theta_2,
    \\
   \left[\widetilde{\nabla} \, {\rm GAP}_q(\theta)\right]_3 & = & 
   \left(\frac{p^\ast(0)}{\theta_1 + \theta_3} - \frac{p^\ast(1)}{\theta_2 + \theta_4}\right) \theta_3 (\theta_2 + \theta_4) - 
 p^\ast(0)q(1|0) + \theta_3.   
\end{eqnarray*}
Mapping the Euclidean gradient with $d\pi_V$ yields 
\begin{eqnarray*}
   d\pi_V \left( \sum_{i = 1}^3 \left[ \nabla \, {\rm GAP}_q(\theta)\right]_i \, \partial_i(\theta) \right) & = & \sum_{i = 1}^3 \left[ \nabla  
   \, {\rm GAP}_q(\theta)\right]_i \, d \pi_V  (\partial_i(\theta)) \\
   & = & \left( 
   \left[ \nabla  
   \, {\rm GAP}_q(\theta)\right]_1
    + 
    \left[ \nabla  
   \, {\rm GAP}_q(\theta)\right]_3
    \right) (\delta^0 - \delta^1).  
\end{eqnarray*}
The same formula holds for the natural gradient where we simply replace $\nabla$ by $\widetilde{\nabla}$. Thus, in both cases we can analyze whether the image of the gradient under $d\pi_V$ vanishes by simply adding the respective first and third components. Let us begin with the natural gradient:
\begin{eqnarray*}
\lefteqn{\left[\widetilde{\nabla} \, {\rm GAP}_q(\theta)\right]_1 + \left[\widetilde{\nabla} \, {\rm GAP}_q(\theta)\right]_3 }\\ 
  & = &  \left(\frac{p^\ast(0)}{\theta_1 + \theta_3} - \frac{p^\ast(1)}{\theta_2 + \theta_4}\right) \theta_1 (\theta_2 + \theta_4) - 
 p^\ast(0)q(0|0) + \theta_1  \\
 & & + \left(\frac{p^\ast(0)}{\theta_1 + \theta_3} - \frac{p^\ast(1)}{\theta_2 + \theta_4}\right) \theta_3 (\theta_2 + \theta_4) - 
 p^\ast(0)q(1|0) + \theta_3 \\
 & = &  \left(\frac{p^\ast(0)}{\theta_1 + \theta_3} - \frac{p^\ast(1)}{\theta_2 + \theta_4}\right) (\theta_1 + \theta_3) (\theta_2 + \theta_4)
 \underbrace{- p^\ast(0)q(0|0) - p^\ast(0)q(1|0)}_{- p^\ast(0)} +  \theta_1 + \theta_3 \\
 & = & p^\ast(0) (\theta_2 + \theta_4) - p^\ast(1) (\theta_1 + \theta_3)- p^\ast(0) + \theta_1 + \theta_3 \\
 & = & p^\ast(0) \underbrace{- p^\ast(0) (\theta_1 + \theta_3) - p^\ast(1) (\theta_1 + \theta_3)}_{-(\theta_1 + \theta_3)}- p^\ast(0) + \theta_1 + \theta_3 \\
 & = & 0.
\end{eqnarray*}
This exemplifies our core result (\ref{vanishgap}) in terms of local coordinates. The same calculation for the Euclidean gradient does not lead to this result. Thus, generically we have  
\[
   \left[{\nabla} \, {\rm GAP}_q(\theta)\right]_1 + \left[{\nabla} \, {\rm GAP}_q(\theta)\right]_3 \; \not= \; 0. 
\]
\section{Proof of Proposition \ref{orthogonalVectors}}\label{sec:proof_proposition_ortho_vectors}
\begin{proof}[Proof of Proposition \ref{orthogonalVectors}]
     Without loss of generality, we identify the unit set $N$ with the set $\{1,\dots, n\}$, $n = |N|$, in a way that is consistent with the graph $G = (N,E)$. That means, whenever $i \in {\rm pa}(s)$ we have $i < s$. Note that such an identification is always possible for a directed acyclic graph. Furthermore, we  assume $s < t$. Then:
      \begin{eqnarray*} 
 \lefteqn{g_\theta^{\rm FR}
        \left( \partial_{s,k}(\theta) , \partial_{t,l}(\theta) \right) } \\
        & = &  \sum_x \frac{1}{p(x;\theta)}  \partial_{s,k}(x; \theta) \, 
        \partial_{t,l}(x ; \theta) \qquad (\mbox{by (\ref{FRmetric})}) \\
          & = & \sum_x \frac{1}{p(x;\theta)} 
          \qquad\qquad\qquad\qquad\quad\;\;\, (\mbox{by (\ref{tangentvec})}) \\ 
          & & \times \left(p(x; \theta) \der{\theta_{s,k}} \ln p(x_s|x_{\pa(s)}; \theta_s)\right) \left( p(x; \theta) \der{\theta_{t,l}} \ln p(x_t|x_{\pa(t)}; \theta_t) \right) \\
          & = &\sum_x  p(x; \theta) \,  \der{\theta_{s,k}} \ln p(x_s|x_{\pa(s)}; \theta_s)  \der{\theta_{t,l}} \ln p(x_t|x_{\pa(t)}; \theta_t) \\
          % & = & \sum_x \der{\theta_{s,k}} \ln p(x_s|x_{\pa(s)}; \theta_s) \prod_{i=1}^t  p(x_i|x_{\pa(i)};\theta_i) \der{\theta_{t,l}}\ln p(x_t|x_{\pa(t)}; \theta_t) \\
          % & & \times \underbrace{\prod_{i=t+1}^n p(x_i|x_{\pa(i)};\theta_i)}_{=\, 1} \\
          % & = & \sum_{x_1, ..., x_t} \der{\theta_{s,k}} \ln p(x_s|x_{\pa(s)}; \theta_s) \prod_{i=1}^t  p(x_i|x_{\pa(i)};\theta_i) \der{\theta_{t,l}}\ln p(x_t|x_{\pa(t)}; \theta_t)\\
          % & = & \sum_{x_1, ..., x_{t-1}} \der{\theta_{s,k}} \ln p(x_s|x_{\pa(s)}; \theta_s) \sum_{x_t} \prod_{i=1}^t  p(x_i|x_{\pa(i)};\theta_i) \\
          % & & \times \der{\theta_{t,l}}\ln p(x_t|x_{\pa(t)}; \theta_t) \\
          % & = &  \sum_{x_1, ..., x_{t-1}} \der{\theta_{s,k}} \ln p(x_s|x_{\pa(s)}; \theta_s) \prod_{i=1}^{t-1}  p(x_i|x_{\pa(i)};\theta_i) \sum_{x_t} p(x_t|x_{\pa(t)};\theta_t) \\
          % & & \times \der{\theta_{t,l}}  \ln p(x_t|x_{\pa(t)}; \theta_t) \\
          % & = & \sum_{x_1, ..., x_{t-1}} \der{\theta_{s, k}} \ln p(x_s|x_{\pa(s)}; \theta_s) \prod_{i=1}^{t-1}  p(x_i|x_{\pa(i)};\theta_i) \der{\theta_{t,l}} \sum_{x_t}  p(x_t|x_{\pa(t)}; \theta_t) \\
          % & = & \sum_{x_1, ..., x_{t-1}} \der{\theta_{s,k}} \ln p(x_s|x_{\pa(s)}; \theta_s) \prod_{i=1}^{t-1}  p(x_i|x_{\pa(i)};\theta)  \der{\theta_{t,l}} \, 1 \\
          % & = & 0 .
      \end{eqnarray*}
      \begin{eqnarray*} 
        % \lefteqn{g_\theta^{\rm FR}
        % \left( \partial_{s,k}(\theta) , \partial_{t,l}(\theta) \right) } \\
        % & = &  \sum_x \frac{1}{p(x;\theta)}  \partial_{s,k}(x; \theta) \, 
        % \partial_{t,l}(x ; \theta) \qquad (\mbox{by (\ref{FRmetric})}) \\
        %   & = & \sum_x \frac{1}{p(x;\theta)} 
        %   \qquad\qquad\qquad\qquad\quad\;\;\, (\mbox{by (\ref{tangentvec})}) \\ 
          % & & \times \left(p(x; \theta) \der{\theta_{s,k}} \ln p(x_s|x_{\pa(s)}; \theta_s)\right) \left( p(x; \theta) \der{\theta_{t,l}} \ln p(x_t|x_{\pa(t)}; \theta_t) \right) \\
          % & = &\sum_x  p(x; \theta) \,  \der{\theta_{s,k}} \ln p(x_s|x_{\pa(s)}; \theta_s)  \der{\theta_{t,l}} \ln p(x_t|x_{\pa(t)}; \theta_t) \\
          & = & \sum_x \der{\theta_{s,k}} \ln p(x_s|x_{\pa(s)}; \theta_s) \prod_{i=1}^t  p(x_i|x_{\pa(i)};\theta_i) \der{\theta_{t,l}}\ln p(x_t|x_{\pa(t)}; \theta_t) \\
          & & \times \underbrace{\prod_{i=t+1}^n p(x_i|x_{\pa(i)};\theta_i)}_{=\, 1} \\
          & = & \sum_{x_1, ..., x_t} \der{\theta_{s,k}} \ln p(x_s|x_{\pa(s)}; \theta_s) \prod_{i=1}^t  p(x_i|x_{\pa(i)};\theta_i) \der{\theta_{t,l}}\ln p(x_t|x_{\pa(t)}; \theta_t)\\
          & = & \sum_{x_1, ..., x_{t-1}} \der{\theta_{s,k}} \ln p(x_s|x_{\pa(s)}; \theta_s) \sum_{x_t} \prod_{i=1}^t  p(x_i|x_{\pa(i)};\theta_i) \\
          & & \times \der{\theta_{t,l}}\ln p(x_t|x_{\pa(t)}; \theta_t) \\
          & = &  \sum_{x_1, ..., x_{t-1}} \der{\theta_{s,k}} \ln p(x_s|x_{\pa(s)}; \theta_s) \prod_{i=1}^{t-1}  p(x_i|x_{\pa(i)};\theta_i) \sum_{x_t} p(x_t|x_{\pa(t)};\theta_t) \\
          & & \times \der{\theta_{t,l}}  \ln p(x_t|x_{\pa(t)}; \theta_t) \\
          & = & \sum_{x_1, ..., x_{t-1}} \der{\theta_{s, k}} \ln p(x_s|x_{\pa(s)}; \theta_s) \prod_{i=1}^{t-1}  p(x_i|x_{\pa(i)};\theta_i) \der{\theta_{t,l}} \sum_{x_t}  p(x_t|x_{\pa(t)}; \theta_t) \\
          & = & \sum_{x_1, ..., x_{t-1}} \der{\theta_{s,k}} \ln p(x_s|x_{\pa(s)}; \theta_s) \prod_{i=1}^{t-1}  p(x_i|x_{\pa(i)};\theta)  \der{\theta_{t,l}} \, 1 \\
          & = & 0 .
      \end{eqnarray*}
\end{proof}

\bibliography{refs}
\end{document}